\documentclass[11pt]{article}

\usepackage{fullpage}
\usepackage[round]{natbib}

\usepackage{amsmath,amsthm,amsfonts,amssymb}
\usepackage{amsmath}
\usepackage{hyperref}
\usepackage{color}
\usepackage{mathrsfs}
\usepackage{bm}
\usepackage{multirow}
\usepackage{booktabs}
\usepackage{makecell}
\usepackage{graphicx}
\usepackage{caption}
\usepackage{subcaption}
\usepackage{thmtools}
\usepackage{thm-restate}
\usepackage{setspace}
\usepackage{makecell}
\definecolor{light-gray}{gray}{0.85}
\usepackage{colortbl}
\usepackage{xcolor}
\usepackage{hhline}

\usepackage[round]{natbib}
\usepackage{hyperref}       
\usepackage{url}    
\usepackage{booktabs}       
\usepackage{nicefrac}       
\usepackage{microtype}  

\usepackage{amsmath,amsthm,amsxtra,graphicx,verbatim,epsfig,color,enumerate,array,mathtools,dsfont,multicol,mathrsfs}

\usepackage[ruled,linesnumbered,vlined]{algorithm2e}
\usepackage{algpseudocode}
\SetKwInOut{Parameter}{parameter}

\hypersetup{
    colorlinks,
    linkcolor={blue!50!black},
    citecolor={blue!50!black},
}
\colorlet{linkequation}{blue}

  {\list{}{\leftmargin=0.3in\rightmargin=0.3in}\item[]}%
  {\endlist}
  
\usepackage[ruled,linesnumbered,vlined]{algorithm2e}
\usepackage{algpseudocode}
\SetKwInOut{Parameter}{parameter}

\usepackage{amsfonts}

\newtheorem{theorem}{Theorem}[section]
\newtheorem{lemma}[theorem]{Lemma}
\newtheorem{corollary}[theorem]{Corollary}
\newtheorem{remark}[theorem]{Remark}

\theoremstyle{definition}
\newtheorem{definition}[theorem]{Definition}

\newtheorem{assumption}{Assumption}

\newcommand{\lep}[1]{\mathop  \le \limits^{(#1)}}

\newcommand{\ep}[1]{\mathop  = \limits^{(#1)}}
\newcommand{\ex}[1]{\mathbb{E}\left[ #1 \right] }

\newcommand{\priv}{\rmfamily\scshape Privatizer}

\newcommand{\cA}{\mathcal{A}}

\newcommand{\cE}{\mathcal{E}}

\newcommand{\cM}{\mathcal{M}}
\newcommand{\cN}{\mathcal{N}}

\newcommand{\cR}{\mathcal{R}}
\newcommand{\cS}{\mathcal{S}}

\newcommand{\cU}{\mathcal{U}}

\newcommand{\cX}{\mathcal{X}}

\newcommand{\Real}{\mathbb{R}}
\newcommand{\Nat}{\mathbb{N}}

\newcommand{\indic}[1]{\mathbb{I}\{#1\}}

\newcommand{\abs}[1]{\left|#1\right|}
\newcommand{\norm}[1]{\left\lVert#1\right\rVert}

\newcommand{\prob}[1]{\mathbb{P}\left[{#1}\right]}

\newcommand{\inner}[2]{\langle #1, #2 \rangle}

\newcommand{\argmin}{\mathop{\mathrm{argmin}}}

\newcommand{\beq}{\begin{equation}}
\newcommand{\eeq}{\end{equation}}
\newcommand{\beqn}{\begin{equation*}}
\newcommand{\eeqn}{\end{equation*}}
\newcommand{\beqa}{\begin{eqnarray}}
\newcommand{\eeqa}{\end{eqnarray}}
\newcommand{\beqan}{\begin{eqnarray*}}
\newcommand{\eeqan}{\end{eqnarray*}}

\renewcommand{\epsilon}{\varepsilon}

\newcommand{\ra}{\rightarrow}

\DeclareMathOperator{\polylog}{polylog}

\newcommand{\local}{\rmfamily\scshape Local-Privatizer}
\newcommand{\central}{\rmfamily\scshape Central-Privatizer}

\newcommand{\loc}{\rmfamily\scshape Local}
\newcommand{\cen}{\rmfamily\scshape Central}

\newcommand{\po}{\rmfamily\scshape Private-UCB-PO}
\newcommand{\vi}{\rmfamily\scshape Private-UCB-VI}
\newcommand{\pucb}{\rmfamily\scshape PUCB}
\newcommand{\ldpobi}{\rmfamily\scshape LDP-OBI}

\newcommand{\oppo}{\rmfamily\scshape OPPO}
\newcommand{\ucbvi}{\rmfamily\scshape UCB-VI}

\newcommand{\psum}{\ensuremath{\mathrm{P}\text{-}\mathrm{sum}}}
\newcommand{\psums}{\ensuremath{\mathrm{P}\text{-}\mathrm{sums}}}

\usepackage{times}

\begin{document}

\title{Differentially Private Regret Minimization in Episodic Markov Decision Processes}

\author {
    Sayak Ray Chowdhury\thanks{Equal contributions} \footnote{Indian Institute of Science, Bangalore, India. Email: \texttt{sayak@iisc.ac.in}  }\quad
    Xingyu Zhou\footnotemark[1] \footnote{Wayne State University, Detroit, USA.  Email: \texttt{xingyu.zhou@wayne.edu}}
}

\date{}

\maketitle

\begin{abstract}
We study regret minimization in finite horizon tabular Markov decision processes (MDPs) under the constraints of differential privacy (DP). This is motivated by the widespread applications of reinforcement learning (RL) in real-world sequential decision making problems, where protecting users' sensitive and private information is becoming paramount. We consider two variants of DP -- joint DP (JDP), where a centralized agent is responsible for protecting users' sensitive data and local DP (LDP), where information needs to be protected directly on the user side.
We first propose two general frameworks -- one for policy optimization and another for value iteration -- for designing private, optimistic RL algorithms. We then instantiate these frameworks with suitable privacy mechanisms to satisfy JDP and LDP requirements, and simultaneously obtain sublinear regret guarantees. The regret bounds show that under JDP, 
the cost of privacy is only a lower order additive term, while for a stronger privacy protection under LDP, the cost suffered is multiplicative.
Finally, the regret bounds are obtained by a unified analysis, which, we believe, can be extended beyond tabular MDPs.
\end{abstract}

\section{Introduction}
Reinforcement learning (RL) is a fundamental sequential decision making problem, where an agent learns to maximize its reward in an unknown environment through trial and error. Recently, it is ubiquitous in various personalized services, including healthcare~\citep{gottesman2019guidelines}, virtual assistants~\citep{li2016deep}, social robots~\citep{gordon2016affective} and online recommendations~\citep{li2010contextual}. In these applications, the learning agent 
continuously improves its decision by learning from users' personal data and feedback. However, nowadays people are becoming increasingly concerned about potential privacy leakage in these interactions. For example, in personalized healthcare, the private data of a patient can be sensitive informations such as her age, gender,  height, weight, medical history, state of the treatment, etc. Therefore, developing RL algorithms which can protect users' private data are of paramount importance in these applications.

\begin{table*}[!ht]
\begin{center}
\begin{tabular}{ | c | c| c |  c |} 
\hline
& \textbf{Algorithm} & \textbf{Regret ($\epsilon$-JDP)} & \textbf{Regret ($\epsilon$-LDP)}\\ 
\hline
PO & \cellcolor[gray]{0.8}{\po} & $\widetilde{O}\left(\sqrt{S^2AH^3T} +  S^2AH^3/\epsilon\right)$ & $\widetilde{O}\left(\sqrt{S^2AH^3T} +  S^2A\sqrt{H^5T}/\epsilon\right)$ \\ 
\hline
\multirow{3}{1em}{VI} & \cellcolor[gray]{0.8}{\vi} & $\widetilde{O}\left(\sqrt{SAH^3T} +  S^2AH^3/\epsilon\right)$ &$\widetilde{O}\left(\sqrt{SAH^3T} +  S^2A\sqrt{H^5 T}/\epsilon\right)$ \\\cline{2-4}
& {\pucb}~\citep{vietri2020private} & $\widetilde{O}\left(\sqrt{S^2AH^3T} +  S^2AH^3/\epsilon\right)$ \footnotemark & NA  \\ \cline{2-4}
& {\ldpobi}~\citep{garcelon2020local} & NA & $\widetilde{O}\left(\sqrt{S^2AH^3T} +  S^2A\sqrt{H^5 T}/\epsilon\right)$\footnotemark   \\ 
\hline
\end{tabular}
\end{center}
\caption{ \footnotesize{Regret comparisons for private RL algorithms on episodic tabular MDP. $T = KH$ is total number of steps, where $K$ is the total number of episodes and $H$ is the
number of steps per episode. $S$ is the number of states, and $A$ is the number of actions. $\epsilon > 0$ is the desired privacy level. $\widetilde O(\cdot)$ hides $\polylog\left(S,A,T,1/\delta \right)$ factors, where $\delta \in (0,1]$ is the desired confidence level.}} 
\label{tab:comp}
\vspace{-5mm}
\end{table*}

\emph{Differential privacy} (DP)~\citep{dwork2008differential} has become a standard in designing private sequential decision-making algorithms both in the full information ~\citep{jain2012differentially} and partial or bandit information ~\citep{mishra2015nearly,tossou2016algorithms} settings. Under DP, the learning agent collects users' raw data to train its algorithm while ensuring that its output will not reveal users' sensitive information. This notion of privacy protection is suitable for situations, where a user is willing to share her own information to the agent in order to obtain a service specially tailored to her needs, but meanwhile she does not like to allow any third party to infer her private information seeing the output of the  learning algorithm (e.g., Google GBoard).
However, a recent body of work~\citep{shariff2018differentially,dubey2021no} show that the standard DP guarantee is irreconcilable with sublinear regret in contextual bandits, and thus, a variant of DP, called \emph{joint differential privacy} (JDP)~\citep{kearns2014mechanism} is considered. Another variant of DP, called \emph{local differential privacy} (LDP)~\citep{duchi2013local} has recently gained increasing popularity in personalized services due to its stronger privacy protection. It has been studied in various bandit settings recently~\citep{ren2020multi,zheng2020locally,zhou2020local}. Under LDP, each user's raw data is directly protected before being sent to the learning agent. Thus, the learning agent only has access to privatized data to train its algorithm, which often leads to a worse regret guarantee compared to DP or JDP. 

In contrast to the vast amount of work in private bandit algorithms, much less attention are given to address privacy in RL problems. To the best of our knowledge,~\citet{vietri2020private} propose the first RL algorithm -- {\pucb} -- for regret minimization with JDP guarantee in tabular finite state, finite action MDPs. On the other hand,~\citet{garcelon2020local} design the first private RL algorithm -- {\ldpobi} -- with regret and LDP guarantees. Recently, \citet{chowdhury2021adaptive} study linear quadratic regulators under the JDP constraint. It is worth noting that all these prior work consider only value-based RL algorithms, and a study on policy-based private RL algorithms remains elusive. Recently, policy optimization (PO) has seen great success in many real-world applications, especially when coupled with deep neural networks~\citep{silver2017mastering,duan2016benchmarking,wang2018deep}, and a variety of PO based algorithms have been proposed~\citep{williams1992simple,kakade2001natural,schulman2015trust,schulman2017proximal,konda2000actor}. The theoretical understandings of PO have also been studied in both computational (i.e., convergence) perspective~\citep{liu2019neural,wang2019neural} and statistical (i.e., regret) perspective~\citep{cai2020provably,efroni2020optimistic}. Thus, one fundamental question to ask is how to build on existing understandings of non-private PO algorithms to design sample-efficient policy-based RL algorithms with general privacy guarantees (e.g., JDP and LDP), which is the main motivation behind this work. Also, the existing regret bounds in both~\citet{vietri2020private} and~\citet{garcelon2020local} for private valued-iteration (VI) based RL algorithms are loose. Moreover, the algorithm design and regret analysis under JDP in~\citet{vietri2020private} and the ones under LDP in~\citet{garcelon2020local} follow different approaches (e.g., choice of exploration bonus terms and corresponding analysis). Thus, another important question to ask is whether one can obtain tighter regret bounds for VI based private RL algorithms via a unified framework under general privacy requirements.
\addtocounter{footnote}{-1}\footnotetext{\citet{vietri2020private} claim a $\widetilde{O}\left(\sqrt{SAH^3T} +  S^2AH^3/\epsilon\right)$ regret bound for {\pucb}. However, to the best of our understanding, we believe the current analysis has gaps (see Section~\ref{sec:vi}), and the best achievable regret for {\pucb} should have an additional $\sqrt{S}$ factor in the first term.}
\addtocounter{footnote}{+1}\footnotetext{\citet{garcelon2020local} consider stationary transition kernels, and show a $\widetilde{O}\left(\sqrt{S^2AH^2T} +  S^2A\sqrt{H^5 T}/\epsilon\right)$
regret bound for {\ldpobi}. For non-stationary transitions, as considered in this work, an additional multiplicative $\sqrt{H}$ factor would appear in the first term of the bound.
}



\textbf{Contributions.} Motivated by the two questions above, we make the following contributions. \begin{itemize}
\item 
We present a general framework -- {\po} -- for designing private policy-based optimistic RL algorithms in tabular MDPs. This framework enables us to establish the first regret bounds for PO under both JDP and LDP requirements by instantiating it with suitable private mechanisms -- the {\central} and the {\local} -- respectively. 
\item
We revisit private optimistic value-iteration in tabular MDPs by proposing a general framework -- {\vi} -- for it. This framework allows us to improve upon the existing regret bounds under both JDP and LDP constraints using a unified analysis technique.
\item Our regret bounds show that for both policy-based and value-based private RL algorithms, the cost of JDP guarantee is only a lower-order additive term compared to the non-private regret. In contrast, under the stringer LDP requirement, the cost suffered is multiplicative and is of the same order. Our regret bounds and their comparison to the existing ones is summarised in Table~\ref{tab:comp}.
\end{itemize}







\textbf{Related work.} Beside the papers mentioned above, there are other related work on differentially private online learning~\citep{guha2013nearly,agarwal2017price} and multi-armed bandits~\citep{tossou2017achieving,hu2021optimal,sajed2019optimal,gajane2018corrupt,chen2020locally}. In the RL setting, in addition to~\citet{vietri2020private,garcelon2020local} that focus on value-iteration based regret minimization algorithms under privacy constraints,~\citet{balle2016differentially} considers private policy evaluation with linear function approximation. For MDPs with continuous state spaces, \citet{wang2019privacy} proposes a variant of $Q$-learning to protect the rewards information by directly injecting noise into value functions. Recently, a distributed actor-critic RL algorithm under LDP is proposed in~\citet{ono2020locally} but without any regret guarantee. While there are recent advances in regret guarantees for policy optimization~\citep{cai2020provably,efroni2020optimistic}, we are not aware of any existing work on private policy optimization. Thus, our work takes the first step towards a unified framework for private policy-based RL algorithms in tabular MDPs with general privacy and regret guarantees. 




\section{Problem formulation}

In this section, we recall the basics of episodic
Markov Decision Processes and introduce the
notion of differential privacy in reinforcement learning.

\subsection{Learning model and regret in episodic MDPs}
We consider episodic reinforcement learning (RL) in a finite horizon stochastic Markov decision process (MDP) \citep{puterman94, sutton1988learning} given by a tuple $(\mathcal{S},\mathcal{A},H, (P_h)_{h=1}^H, (c_h)_{h=1}^H)$, where $\mathcal{S}$ and $\mathcal{A}$ are state and action spaces with cardinalities $S$ and $A$, respectively, $H \in \Nat$ is the  episode length, $P_h(s'|s,a)$ is the probability of transitioning to state $s'$ from state $s$
provided action $a$ is taken at step $h$ and $c_h(s,a)$ is the mean of the cost distribution at step $h$ supported on $[0,1]$. The actions are chosen following some policy $\pi=(\pi_h)_{h=1}^{H}$, where each $\pi_h$ is a mapping from the state space $\cS$ into a probability distribution over the action space $\cA$, i.e. $\pi_h(a|s)\ge 0$ and $\sum_{a \in \cA}\pi_h(a|s)=1$ for all $s \in \cS$. The agent would like to find a policy $\pi$ that minimizes the long term expected cost starting from every state $s \in \cS$ and every step $h \in [H]$, defined as
\begin{align*}
    V^{\pi}_h(s):= \mathbb{E}\left[\sum\nolimits_{h'=h}^H c_{h'}(s_{h'}, a_h') \mid s_h = s,\pi\right],
    \end{align*}
where the expectation is with respect to the randomness of
the transition kernel and the policy.
We call $V_{h}^{\pi}$ the value function of policy $\pi$ at step $h$. Now, defining the $Q$-function of policy $\pi$ at step $h$ as
\begin{align*}
    &Q^{\pi}_h(s,a):= \mathbb{E}\left[\sum\nolimits_{h'=h}^H c_{h'}(s_{h'}, a_h')) \mid s_h = s, a_h = a,\pi\right],
    \end{align*}
we obtain $Q^{\pi}_h(s,a) = c_h(s,a)+ \sum_{s' \in \cS}V^{\pi}_{h+1}(s') P_h(s'|s,a)$ and $V_h^\pi(s)=\sum_{a \in \cA} Q^{\pi}_h(s,a) \pi_h(a|s)$.  

A policy $\pi^*$ is said to be optimal if it minimizes the value for all states $s$ and step $h$ simultaneously, and the corresponding optimal value function is denoted by $V^*_{h}(s)=\min_{\pi \in \Pi}V^{\pi}_{h}(s)$ for all $h \in [H]$, where $\Pi$ is the set of all non-stationary policies. The agent interacts with the environment for $K$ episodes to learn the unknown transition probabilities $P_h(s'|s,a)$ and mean costs $ c_h(s,a)$, and thus, in turn, the optimal policy $\pi^*$.
At each episode $k$, the agent chooses a policy $\pi^k=(\pi_h^k)_{h=1}^{H}$ and samples a trajectory $\lbrace s_1^k,a_1^k,c_1^k,\ldots,s_H^k,a_H^k,c_H^k,s_{H+1}^k \rbrace$ by interacting with the MDP using this policy. Here, at a given step $h$, $s_h^k$ denotes the state of the MDP, $a_h^k \sim \pi_h^k(\cdot|s_h^k)$ denotes the action taken by the agent, $c_h^k \in [0,1]$ denotes the (random) cost suffered by the agent with the mean value $c_h(s_h^k,a_h^k)$ and $s_{h+1}^k\sim P_h(\cdot | s_h^k,a_h^k)$ denotes the next state. The initial state $s_1^k$ is assumed to be fixed and history independent. We measure performance of the agent by the cumulative (pseudo) regret accumulated
over $K$ episodes, defined as
\begin{align*}
    R(T) := \sum\nolimits_{k=1}^K\left[ V_1^{\pi^k}(s_1^k)-V_1^{*}(s_1^k)\right],
\end{align*}
where $T=KH$ denotes the total number of steps. We seek
algorithms with regret that is sublinear in $T$, which
demonstrates the agent’s ability to act near optimally.

\subsection{Differential privacy in episodic RL}

In the episodic RL setting described above, it is natural to view each episode $k \in [K]$ as a
trajectory associated to a specific user. To this end, we let $U_K = (u_1, \ldots,u_K) \in \cU^K$ to denote a sequence of $K$ unique\footnote{Uniqueness is assumed wlog, as for a returning user one can group her with her previous occurrences.} users participating in the private RL protocol with an RL agent $\cM$, where $\cU$ is the set of all users. Each user $u_k$ is identified by the cost and state responses $(c_h^k,s_{h+1}^k)_{h=1}^{H}$ she gives to the actions $( a_{h}^k)_{h=1 }^{H}$ chosen by the agent. We let $\cM(U_k)=(a_1^1,\ldots,a_H^K) \in \cA^{KH}$ to denote the set of all actions chosen by the agent $\cM$ when interacting with the user sequence $U_k$.
Informally, we will be interested in (centralized) randomized mechanisms (in this case, RL agents) $\cM$ 
so that the knowledge of the output $\cM(U_k)$ and all but the $k$-th user $u_k$ does not reveal `much' information about $u_k$. 
We formalize this in the following definition.
\begin{definition}[Differential Privacy (DP)]\label{def:DP}
For any $\epsilon \ge 0$ and $\delta \in [0,1]$, a mechanism $\cM : \cU^K \ra \cA^{KH}$ is $(\epsilon,\delta)$-differentially private if for all $U_K,U'_K \in \cU^K$ differing on a single user and for all subset of actions $\cA_0 \subset \cA^{KH}$, 
\beqn
\prob{\cM(U_K)\in \cA_0} \le \exp(\epsilon) \prob{\cM(U'_K)\in \cA_0} + \delta~.
\eeqn
If $\delta=0$, we call the mechanism $\cM$ to be $\epsilon$-differentially private ($\epsilon$-DP).
\end{definition}
This is a direct adaptation of the classic notion of differential privacy \citep{dwork2014algorithmic}. However, we need to relax this definition for our purpose, because
although the actions recommended to the user $u_k$ have only a small
effect on the types (i.e., state and cost responses) of other users participating in the RL protocol, those can reveal a lot of information about the type of the user $u_k$. Thus, it becomes hard to privately recommend the actions to user $u_k$ while protecting the privacy of it's type, i.e., it's state and cost responses to the suggested actions.
Hence, to preserve the privacy of individual users, we consider the notion of \emph{joint differential privacy} (JDP) \citep{kearns2014mechanism}, which requires that simultaneously for all user $u_k$, the
joint distribution of the actions recommended to all users other than $u_k$ be differentially private in the type of the user $u_k$.
It weakens the constraint of DP only in that
the actions suggested specifically to $u_k$ may be sensitive in her type (state and cost responses). However, JDP is
still a very strong definition since it protects $u_k$ from any arbitrary collusion of other users against her, so long as she does
not herself make the actions suggested to her public.
To this end, we let $\cM_{-k}(U_k):=\cM(U_k)\setminus ( a_{h}^k)_{h=1}^{H}$ to denote all the actions chosen by the agent $\cM$ excluding those recommended to $u_k$ and formally define JDP as follows.
\begin{definition}[Joint Differential Privacy (JDP)]\label{def:JDP}
For any $\epsilon \!\ge\! 0$,
a mechanism $\cM \!:\! \cU^K \!\ra\! \cA^{KH}$ is $\epsilon$-joint differentially private if for all $k \!\in\! [K]$, for all user sequences $U_K,U'_K \!\in\! \cU^K$ differing only on the $k$-th user and for all set of actions $\cA_{-k} \!\subset\! \cA^{(K-1)H}$ given to all but the $k$-th user, 
\beqn
\prob{\cM_{-k}(U_K)\in \cA_{-k}} \le \exp(\epsilon) \prob{\cM_{-k}(U'_K)\in \cA_{-k}}.
\eeqn
\end{definition}
JDP has been used extensively in private mechanism design \citep{kearns2014mechanism}, in private matching and allocation problems \citep{hsu2016private}, in designing privacy-preserving algorithms
for linear contextual bandits \citep{shariff2018differentially}, and it has been introduced in private tabular RL by \citet{vietri2020private}.




JDP allows the agent to observe the data (i.e., the entire trajectory of state-action-cost sequence) associated with each user and the privacy burden lies on the agent itself. In some scenarios, however, the users may not even be willing to share it's data with the agent directly. This motivates a stringer notion of privacy protection, called the \emph{local differential privacy} (LDP)~\citep{duchi2013local}. In this setting, each user is assumed to have her own privacy mechanism that can do randomized mapping on its data to guarantee privacy. To this end, we denote by $X$ a trajectory $(s_h,a_h,c_h,s_{h+1})_{h=1}^{H}$ and by $\cX$ the set of all possible trajectories. We write $\cM'(X)$ to denote the privatized trajectory generated by a (local) randomized mechanism $\cM'$. With this notation, we now formally define LDP for our RL protocol.
\begin{definition}[Local Differential Privacy (LDP)]\label{def:LDP}
For any $\epsilon \ge 0$, a mechanism $\cM'$ is $\epsilon$-local differentially private if for all trajectories $X,X' \in \cX$ and for all possible subsets $\cE_0 \subset \lbrace \cM'(X)|X \in \cX\rbrace$,
\beqn
\prob{\cM'(X)\in \cE_0} \le \exp(\epsilon) \prob{\cM'(X')\in \cE_0} .
\eeqn
\end{definition}
LDP ensures that if any adversary (can be the RL agent itself) observes the output of the privacy mechanism $\cM'$ for two different trajectories, then it is statistically difficult for it to guess which output is from which trajectory. This has been used extensively in multi-armed bandits \citep{zheng2020locally,ren2020multi}, and introduced in private tabular RL by \citet{garcelon2020local}.




\begin{algorithm}[t!]
\caption{{\po}}
\label{alg:PO}
\DontPrintSemicolon
\KwIn{Number of episodes $K$, time horizon $H$, privacy level $\epsilon > 0$, a {\priv} ({\loc} or {\cen}), confidence level $\delta \in (0,1]$ and parameter $\eta > 0$}
Initialize policy $\pi_h^1(a|s)=1/A$ for all $(s,a,h)$\;
Initialize private counts $\widetilde{C}_h^1(s,a)=0$, $\widetilde{N}_h^{1}(s,a)=0$ and $\widetilde{N}_h^{1}(s,a,s')=0$ for all $(s,a,s',h)$\;
Set precision levels $E_{\epsilon,\delta,1}, E_{\epsilon,\delta,2}$ of the {\priv} \;
\For{$k=1,2,3,\ldots,K$}{
Initialize private value estimates: $\widetilde{V}^k_{H+1}(s) = 0$ \;
\For{$h=H,H-1,\ldots,1$} {
    Compute $\widetilde{c}_h^k(s,a)$ and $\widetilde{P}_h^k(s,a)$ $\forall (s,a)$ as in~\eqref{eq:priv_est} using the private counts\;
    Set exploration bonus using Lemma~\ref{lem:conc_po}: $\beta_h^k(s,a) = \beta_h^{k,c}(s,a) + H\beta_h^{k,p}(s,a)$ $\forall (s,a)$ \;
    Compute: $\forall (s,a)$, $\widetilde{{Q}}_h^k(s,a) = \min\{H-h+1, \max\{0,\widetilde{c}_h^k(s,a)  + \sum_{s'\in \mathcal{S}} \widetilde{V}_{h+1}^k(s')\widetilde{P}_h^k(s'|s,a) - \beta_h^{k}(s,a)\} \}$ \;
        Compute private value estimates: $\forall s$, $\widetilde{V}_h^k(s) = \sum_{a \in \cA}\widetilde{Q}_h^k(s,a) \pi_h^k(a|s)$ 
}
Roll out a trajectory $(s_1^k,a_1^k,c_1^k,\ldots,s_{H+1}^k)$ by acting the policy $\pi^k=(\pi_h^k)_{h=1}^{H}$\;
Receive private counts $\widetilde{C}_h^{k+1}(s,a)$, $\widetilde{N}_h^{k+1}(s,a)$, $\widetilde{N}_h^{k+1}(s,a,s')$ from the {\priv}\;
Update policy: $\forall (s,a,h)$,
        $\pi_h^{k+1}(a|s) = \frac{\pi_h^k(a|s)\exp(-\eta \widetilde{Q}_h^k(s,a)) }{\sum_{a \in \cA} \pi_h^k(a|s)\exp(-\eta \widetilde{Q}_h^k(s,a))}$
}  
\end{algorithm}



\section{Private Policy Optimization}
\label{sec:po}
In this section, we introduce a policy-optimization based private RL algorithm {\po}  (Algorithm~\ref{alg:PO}) that can be instantiated with any private mechanism (henceforth, referred as a {\priv}) satisfying a general condition. We derive a generic regret bound for {\po}, which can be applied to obtain bounds under JDP and LDP requirements by instantiating {\po} with a {\central} and a {\local}, respectively. All the proofs are deferred to the appendix.


Let us first introduce some notations. We
denote by $N_h^k(s,a):= \sum_{k'=1}^{k-1} \mathbb{I}\{s_h^{k^{\prime}}=s, a_h^{k^{\prime}} = a\}$, the number of times that the agent has
visited state-action pair $(s,a)$ at step $h$ \emph{before} episode $k$. Similarly, $N_h^k(s,a,s'):= \sum_{k'=1}^{k-1} \mathbb{I}\{s_h^{k^{\prime}}=s, a_h^{k^{\prime}} = a, s_{h+1}^{k^{\prime}} = s'\}$ denotes the count of going to state $s'$ from $s$ upon playing action $a$ at step $h$ \emph{before} episode $k$.
Finally, $C_h^k(s,a):= \sum_{k'=1}^{k-1} \mathbb{I}\{s_h^{k^{\prime}}=s, a_h^{k^{\prime}} = a\} c_h^{k^{\prime}}$ denotes the total cost suffered by taking action $a$ on state $s$ and step $h$ \emph{before} episode $k$. In non-private learning, these counters are sufficient to find estimates of the transition kernels $(P_h)_h$ and mean cost functions $(c_h)_h$ to design the policy $(\pi_h^k)_h$ for episode $k$. However, in private learning, the challenge is that the counters depend on users' state and cost responses to suggested actions, which is considered sensitive information. Therefore, the {\priv} must release the counts in a privacy-preserving way on which the learning agent would rely. To this end, we let $\widetilde{N}_h^k(s,a)$, $\widetilde{C}_h^k(s,a)$, and $\widetilde{N}_h^k(s,a,s')$ to denote the privatized versions of $N_h^k(s,a)$, $C_h^k(s,a)$, and $N_h^k(s,a,s')$, respectively.
Now, we make a general assumption on the counts released by the {\priv} (both {\loc} and {\cen}), which roughly means that with high probability the private counts are close to the actual ones. 

\begin{assumption}[Properties of private counts]
\label{ass:rand}
For any $\epsilon > 0$ and $\delta \in (0,1]$, there exist functions $E_{\epsilon,\delta,1}, E_{\epsilon,\delta,2} > 0$ such that
with probability at least $1-\delta$, uniformly over all $(s,a,h,k)$, the private counts returned by the {\priv} (both {\loc} and {\cen}) satisfy: (i) $|\widetilde{N}_{h}^k(s,a) - {N}_h^k(s,a)| \le E_{\epsilon,\delta,1}$, (ii)
     $|\widetilde{C}_h^k(s,a) - {C}_h^k(s,a)| \le E_{\epsilon,\delta,1}$, and (iii) $|\widetilde{N}_h^k(s,a,s') - {N}_h^k(s,a,s')| \le E_{\epsilon,\delta,2}$.
\end{assumption}

In the following, we assume Assumption~\ref{ass:rand} holds. Then, we define, for all $(s,a,h,k)$, the \emph{private} mean empirical costs and \emph{private} empirical transition probabilities as
\begin{equation}
 \begin{split}
\label{eq:priv_est}    
 \widetilde{c}_h^k(s,a)  &:= \frac{\widetilde{C}_h^k(s,a)}{\max\{1, \widetilde{N}_h^k(s,a) + E_{\epsilon,\delta,1}\} }\\ \widetilde{P}_h^k(s'|s,a) &:= \frac{\widetilde{N}_h^k(s,a,s')}{ \max\{1,\widetilde{N}_h^k(s,a) +  E_{\epsilon,\delta,1}\}}~.
\end{split}   
\end{equation}

The following concentration bounds on the private estimates will be the key to our algorithm design.
\begin{lemma}[Concentration of private estimates]
\label{lem:conc_po}
Fix any $\epsilon > 0$ and $\delta \in (0,1]$. Then, under Assumption~\ref{ass:rand}, with probability at least $1-2\delta$, uniformly over all $(s,a,h,k)$,
\begin{align*}
        |c_h(s,a) - \widetilde{c}_h^k(s,a)| \le   \beta_h^{k,c}(s,a),\;\; \norm{P_h(\cdot|s,a) - \widetilde{P}_h^k(\cdot|s,a)}_1 \le \beta_h^{k,p}(s,a),
\end{align*}
where $\beta_h^{k,c}(s,a):=\frac{L_c(\delta)}{ \sqrt{ \max \lbrace 1,\widetilde{N}_h^k(s,a) +   E_{\epsilon,\delta,1}\rbrace} }+\frac{3E_{\epsilon,\delta,1} }{ \max\lbrace 1,\widetilde{N}_h^k(s,a)+  E_{\epsilon,\delta,1 }\rbrace }$, $L_c(\delta) := \sqrt{2\ln\frac{4SAT}{\delta}}$, $\beta_h^{k,p}(s,a):=\frac{L_p(\delta)}{\sqrt{ \max \lbrace 1,\widetilde{N}_h^k(s,a) +   E_{\epsilon,\delta,1}\rbrace}}+ \frac{SE_{\epsilon,\delta,2}+2 E_{\epsilon,{\delta}, 1}}{\max\lbrace 1,\widetilde{N}_h^k(s,a)+  E_{\epsilon,\delta,1 }\rbrace}$, and $L_p(\delta):=\sqrt{4S\ln\frac{6SAT}{\delta}}$.
\end{lemma}

\textbf{{\po} algorithm.} {\po}  (Algorithm~\ref{alg:PO}) is a private policy optimization (PO) algorithm based on the celebrated upper confidence bound (UCB) philosophy \citep{Auer02,jaksch2010near}. Similar to the non-private setting~\citep{efroni2020optimistic}, it basically has two stages at each episode $k$: \emph{policy evaluation} and \emph{policy improvement}. In the policy evaluation stage, it evaluates the policy $\pi^k$ based on $k-1$ historical trajectories. In contrast to the non-private case, {\po} relies only on the private counts (returned by the {\priv}) to calculate the private mean empirical costs and private empirical transitions. These two along with a UCB exploration bonus term (which also depends only on private counts) are used to compute $Q$-function estimates. The $Q$-estimates are then truncated 
and corresponding value estimates are computed by taking their expectation with respect to the policy. Next, a new trajectory is rolled out by acting the policy $\pi^k$ and the {\priv} translates all non-private counts into the private ones to be used for the policy evaluation in the next episode. 
Finally, in the policy improvement stage, {\po} employs a `soft' update of the current policy $\pi^k$ by following a standard mirror-descent step together with a Kullback–Leibler (KL) divergence proximity term \citep{beck2003mirror,cai2020provably,efroni2020optimistic}. 
The following theorem presents a general regret bound of {\po} (Algorithm~\ref{alg:PO}) when instantiated with any {\priv} ({\loc} or {\cen}) that satisfies Assumption~\ref{ass:rand}. 

\begin{theorem}[Regret bound of {\po}]
\label{thm:PO}
Fix any $\epsilon > 0$ and $\delta \in (0,1]$ and set $\eta=\sqrt{2\log A/(H^2K)}$. Then, under Assumption~\ref{ass:rand}, with probability at least $1-\delta$, the cumulative regret of {\po} is 
\begin{align*}
    \mathcal{R}(T) &= \widetilde O\left(\sqrt{S^2AH^3T} + \sqrt{S^3A^2H^4}\right) + \widetilde O\left( E_{\epsilon,\delta,2}S^2AH^2 +  E_{\epsilon,\delta,1}SAH^2\right).
\end{align*}
\end{theorem}
\begin{remark}[Cost of privacy]
Theorem~\ref{thm:PO} shows that regret of {\po} is lower bounded by the regret in non-private setting \citep[Theorem 1]{efroni2020optimistic}, and depends directly on the privacy parameter $\epsilon$ through the permitted precision levels
$E_{\epsilon,\delta,1}$ and $E_{\epsilon,\delta,2}$ of the {\priv}. Thus, choosing $E_{\epsilon,\delta,1},E_{\epsilon,\delta,2}$ appropriately to guarantee JDP or LDP, we can obtain regret bounds under both forms of privacy. The cost of privacy, as we shall see in Section~\ref{sec:privacy}, is lower order than the non-private regret under JDP, and is of the same order under the stringer requirement of LDP.\footnote{The lower order terms scale with $S^2$, which is quite common for optimistic tabular RL algorithms \citep{azar2017minimax,dann2017unifying}.}
\end{remark}

\textit{Proof sketch.} We first decompose the regret as in the non-private setting \citep{efroni2020optimistic}: 
\begin{align*}
    \cR(T)&=\sum\nolimits_{k=1}^K \left(V_1^{\pi^k}(s_1^k) - V_1^{\pi^*}(s_1^k)\right) = \sum\nolimits_{k=1}^k \left(V_1^{\pi^k}(s_1^k) -\widetilde{V}_1^k(s_1^k) + \widetilde{V}_1^k(s_1^k) - V_1^{\pi^*}(s_1^k)\right)\\
    &= \underbrace{\sum\nolimits_{k=1}^k \left(V_1^{\pi^k}(s_1^k)-\widetilde{V}_1^k(s_1^k)\right)}_{\mathcal{T}_1}  +\underbrace{\sum\nolimits_{k=1}^K\sum\nolimits_{h=1}^H \ex{ \inner{\widetilde{Q}_h^k(s_h,\cdot)}{\pi_h^k(\cdot|s_h) - \pi_h^*(\cdot|s_h)}| s_1^k,\pi^*} }_{\mathcal{T}_2}\\
    &\quad\quad+\underbrace{\sum\nolimits_{k=1}^K\sum\nolimits_{h=1}^H \ex{ \widetilde{Q}_h^k(s_h,a_h) - c_h(s_h,a_h) - P_h(\cdot|s_h,a_h)\widetilde{V}_{h+1}^k|s_1^k, \pi^*} }_{\mathcal{T}_3}.
\end{align*}
We then bound each of the three terms. In particular, $\mathcal{T}_2$ and $\mathcal{T}_3$ have the same bounds as in the non-private case. Specifically, by setting $\eta = \sqrt{2\log A/(H^2K)}$, we can show that $\mathcal{T}_2 \le \sqrt{2H^4 K \log A}$ via a standard online mirror descent analysis, because $\widetilde{Q}_h^k \in [0,H]$ by design. Furthermore, due to Lemma~\ref{lem:conc_po} and our choice of bonus terms, we have $\mathcal{T}_3 \le 0$. The key is to bound $\mathcal{T}_1$. By the update rule of $Q$-estimate and the choice of bonus terms in {\po}, we can bound $\mathcal{T}_1$ by the sum of expected bonus terms, i.e., 
\begin{align*}
    \mathcal{T}_1  \le \underbrace{\sum\nolimits_{k=1}^K\sum\nolimits_{h=1}^H\ex{ 2\beta_h^{k,c}(s_h,a_h)|s_1^k, \pi^k}}_{\text{Term(i)}} + \underbrace{H\sum\nolimits_{k=1}^K\sum\nolimits_{h=1}^H\ex{ 2\beta_h^{k,p}(s_h,a_h)|s_1^k, \pi^k}}_{\text{Term(ii)}}
\end{align*}
Now, by Assumption~\ref{ass:rand} and the definition of exploration bonus $\beta_h^{k,c}(s,a)$ in Lemma~\ref{lem:conc_po}, we have 
\begin{align*}
    \text{Term(i)} &\le  \sum\nolimits_{k=1}^K\sum\nolimits_{h=1}^H\ex{ \frac{L_c(\delta)}{\sqrt{\max\lbrace{N}_h^k(s_h,a_h),1\rbrace} }+\frac{3E_{\epsilon,\delta,1} }{\max\lbrace N_h^k(s_h,a_h),1\rbrace  }|s_1^k,\pi^k}.
\end{align*}
Note the presence of an additive privacy dependent term. Similarly, we obtain
\begin{align*}
    \text{Term(ii)} \le H\sum\nolimits_{k=1}^K\sum\nolimits_{h=1}^H \ex{\frac{L_p(\delta)}{\sqrt{\max\lbrace N_h^k(s_h,a_h),1\rbrace}} + \frac{SE_{\epsilon,\delta,2}+2 E_{\epsilon,\delta,1}}{\max\lbrace{N}_h^k(s_h,a_h),1\rbrace} | s_1^k, \pi^k}.
\end{align*}
We bound $\text{Term(i)}$ and $\text{Term(ii)}$ by showing that
\begin{align*}
    \sum\nolimits_{k=1}^K\sum\nolimits_{h=1}^H \ex{ \frac{1}{\max\{1,N_h^k(s_h,a_h)\}}|\mathcal{F}_{k-1}} &= O\left(SAH\log T + H\log(H/\delta)\right),\\
    \sum\nolimits_{k=1}^K\sum\nolimits_{h=1}^H \ex{ \frac{1}{\sqrt{\max\{1,N_h^k(s_h,a_h)\}} }|\mathcal{F}_{k-1}} &= O\left(\sqrt{SAHT} + SAH\log T + H\log(H/\delta)\right)
\end{align*}
with high probability. These are generalization of results proved under stationary transition model  \citep{efroni2019tight,zanette2019tighter} to our non-stationary setting (similar results were stated in~\citep{efroni2020exploration,efroni2020optimistic}, but without proofs).
Finally, putting everything together, we complete the proof.
\qed

\section{Private UCB-VI Revisited}
\label{sec:vi}

In this section, we turn to investigate value-iteration based private RL algorithms. 
It is worth noting that private valued-based RL algorithms have been studied under both JDP and LDP requirements \citep{vietri2020private,garcelon2020local}. However, to the best of our understanding, the regret analysis of the JDP algorithm presented in \citet{vietri2020private} has gaps and does not support the claimed result.\footnote{The gap lies in \citet[Lemma 18]{vietri2020private} in which the private estimates were incorrectly used as the true cost and transition functions. This 
lead to a simpler but incorrect regret decomposition since it omits the `error' term between the private estimates and true values. Moreover, the error term cannot be simply upper bounded by its current bonus term ($\widetilde{\text{conf}}_t$ in \citet[Algorithm 3]{vietri2020private}) since one cannot directly use Hoeffding's inequality due to the fact that the value function is not fixed in this term (please refer to Appendix for more detailed discussions).} Under LDP, the regret bound presented in \citet{garcelon2020local} is sub-optimal
in the cardinality of the state space and as the authors have remarked, it is possible to achieve the optimal scaling using a refined analysis. 
Motivated by this, we revisit private value iteration by designing an optimistic algorithm {\vi} (Algorithm~\ref{alg:UCB-VI}) that can be instantiated with a {\priv} ({\cen} and {\loc}) to achieve both JDP and LDP. 

\textbf{{\vi} algorithm.} Our algorithm design principle is again based on the UCB philosophy, the private estimates defined in \eqref{eq:priv_est} and a value-aware concentration result for the estimates stated in Lemma~\ref{lem:conc_vi} below. Similar to the non-private setting~\citep{azar2017minimax}, {\vi} (Algorithm~\ref{alg:UCB-VI}) follows the procedure of optimistic value iteration. Specifically, at each episode $k$, using the private counts and a private UCB bonus term, it first compute private $Q$-estimates and value estimates using optimistic Bellman recursion. Next, a greedy policy $\pi^k$ is obtained directly from the estimated $Q$-function. Finally, a trajectory is rolled out by acting the policy $\pi^k$ and then {\priv} translates all non-private statistics into private ones to be used in the next episode.
\begin{lemma}[Refined concentration of private estimates]
\label{lem:conc_vi}
Fix any $\epsilon > 0$ and $\delta \in (0,1)$. Then, under Assumption~\ref{ass:rand}, with probability at least $1-3\delta$, uniformly over all $(s,a,s',h,k)$,
     \begin{align*}
     |c_h(s,a) - \widetilde{c}_h^k(s,a)| &\le \beta_h^{k,c}(s,a),\\
     \left|(\widetilde{P}_h^k - P_h) {V}_{h+1}^*(s,a)\right| &\le \beta_h^{k,pv}(s,a),\\
     |P_h(s'|s,a) - \widetilde{P}_h^k(s'|s,a)| &\le C\sqrt{\frac{L'(\delta)P_h(s'|s,a)}{\max \lbrace 1,\widetilde{N}_h^k(s,a) +  E_{\epsilon,\delta,1}\rbrace}} +\frac{C L'(\delta)+2E_{\epsilon,\delta,1}+E_{\epsilon,\delta,2}}{\max\lbrace 1,\widetilde{N}_h^k(s,a)+E_{\epsilon,\delta,1}\rbrace},
\end{align*}
where $\beta_h^{k,c}(s,a)$ and $L_c(\delta)$ is as defined in Lemma~\ref{lem:conc_po}, $(P V_{h+1})(s,a):=\sum_{s'} P(s'|s,a) V_{h+1}(s')$, $\beta_h^{k,pv}(s,a):= \frac{H L_c(\delta)}{ \sqrt{ \max\{1,\widetilde{N}_h^k(s,a) +   E_{\epsilon,\delta,1}\}}}  +  \frac{H (S E_{\epsilon,\delta,2} + 2 E_{\epsilon,\delta,1})}{\max \lbrace 1,\widetilde{N}_h^k(s,a) +  E_{\epsilon,\delta,1}\rbrace}$, $C > 0$ is some constant,  and $L'(\delta) := \log\left(\frac{6SAT}{\delta}\right)$.
\end{lemma}
The bonus term $\beta_h^{k,pv}$ in {\vi} does not have the factor $\sqrt{S}$ in the leading term compared to $\beta_h^{k,p}$ in {\po}. This is achieved by following a similar idea in {\ucbvi}~\citep{azar2017minimax}. That is, instead of bounding the transition dynamics as in Lemma~\ref{lem:conc_po}, we maintain a confidence bound directly over the optimal value function (the second result in Lemma~\ref{lem:conc_vi}). Due to this, we have an extra term in the regret bound, which can be carefully bounded by using a Bernstein-type inequality (the third result in Lemma~\ref{lem:conc_vi}). 
These two steps enable us to obtain an improved dependence on $S$ in the regret bound compared to existing private value-based algorithms~\citep{vietri2020private,garcelon2020local} under both JDP and LDP. 
This is stated formally in the next theorem, which presents a general regret bound of {\vi} (Algorithm~\ref{alg:UCB-VI}) when instantiated with any {\priv} ({\loc} or {\cen}).
\begin{theorem}[Regret bound for {\vi}]
\label{thm:VI}
Fix any $\epsilon > 0$ and $\delta \in (0,1]$. Then, under Assumption~\ref{ass:rand}, with probability $\ge 1-\delta$, the regret of {\vi} is 
\begin{align*}
     \mathcal{R}(T) &= \widetilde O \left(\sqrt{SAH^3T} + S^2AH^3\right)+\widetilde O\left( S^2AH^2E_{\epsilon,\delta,1} + S^2AH^2E_{\epsilon,\delta,2}\right).
\end{align*}
\end{theorem}
\begin{remark}[Cost of privacy]
Similar to {\po}, the regret of {\vi} is lower bounded by the regret in non-private setting (see \citet[Theorem 1]{azar2017minimax}),\footnote{In the non-private setting, \citet{azar2017minimax} assume stationary transition kernels $P_h=P$ for all $h$. We consider non-stationary kernels, which adds a multiplicative $\sqrt{H}$ factor in our non-private regret.} and the privacy parameter appear only in the lower order terms.
\end{remark}
\begin{remark}[VI vs. PO]
The regret bound of {\vi} is a $\sqrt{S}$ factor better in the leading privacy-independent term compared to {\po}. This follows the same pattern as in the non-private case, i.e., {\ucbvi} \citep{azar2017minimax} vs. {\oppo} \citep{efroni2020optimistic}. 
\end{remark}

\begin{algorithm}[t!]
\caption{{\vi}}
\label{alg:UCB-VI}
\DontPrintSemicolon
\KwIn{Number of episodes $K$, time horizon $H$, privacy level $\epsilon > 0$, a {\priv} ({\loc} or {\cen}) and confidence level $\delta \in (0,1]$}
Initialize private counts $\widetilde{C}_h^1(s,a)=0$, $\widetilde{N}_h^{1}(s,a)=0$ and $\widetilde{N}_h^{1}(s,a,s')=0$ for all $(s,a,s',h)$\;
Set precision levels $E_{\epsilon,\delta,1},E_{\epsilon,\delta,2}$ of the {\priv}\;
\For{$k=1,\ldots, K$}{
Initialize private value estimates: $\widetilde{V}^k_{H+1}(s) = 0$ \;
\For{$h=H,H-1,\ldots, 1$}{
Compute $\widetilde{c}_h^{k}(s,a)$ and $\widetilde{P}_h^{k}(s'|s,a)$ $\forall (s,a,s')$ as in~\eqref{eq:priv_est} using the private counts\;
    Set exploration bonus using Lemma~\ref{lem:conc_vi}: $\beta_h^k(s,a) =\beta_h^{k,c}(s,a) + \beta_h^{k,pv}(s,a)$ $\forall (s,a)$\;
    Compute: $\forall (s,a)$, $\widetilde{{Q}}_h^k(s,a) = \min\{H-h+1, \max\{0,\widetilde{c}_h^k(s,a) + \sum_{s'\in \mathcal{S}} \widetilde{V}_{h+1}^k(s')\widetilde{P}_h^k(s'|s,a) - \beta_h^k(s,a)\} \}$\;
    Compute private value function: $\forall s$,  $\widetilde{V}_h^k(s) = \min_{a \in \cA} \widetilde{Q}_h^k(s,a)$
  
}
  Compute policy: $\forall (s,h)$, $\pi_h^k(s) = \argmin_{a \in \cA} \widetilde{Q}_h^k(s,a)$ (with breaking ties arbitrarily)\;
Roll out a trajectory $(s_1^k,a_1^k,c_1^k,\ldots,s_{H+1}^k)$ by acting the policy $\pi^k=(\pi_h^k)_{h=1}^{H}$\;
Receive private counts $\widetilde{C}_h^{k+1}(s,a)$, $\widetilde{N}_h^{k+1}(s,a)$, $\widetilde{N}_h^{k+1}(s,a,s')$ from the {\priv}
}
\end{algorithm}

\section{Privacy and regret guarantees}
\label{sec:privacy}

In this section, we instantiate {\po} and {\vi} using a {\central} and a {\local}, and derive corresponding privacy and regret guarantees.



\subsection{Achieving JDP using {\central}}

The {\central} runs a private $K$-bounded binary-tree mechanism (counter) \citep{chan2010private} for each count $N_h^k(s,a), C_h^k(s,a), N_h^k(s,a,s')$, i.e. it uses $2SAH + 
S^2AH$ counters in
total. Let us focus on the counters -- there are $SAH$ many of thems -- for the number of visited states $N_h^k(s,a)$. Each counter takes as input the data stream $\sigma_h(s,a)\in \lbrace 0,1\rbrace^K$, where the $j$-th bit $\sigma_h^j(s,a):=\indic{s_h^j=s,a_h^j=a}$ denotes whether the pair $(s,a)$ is encountered or not at step $h$ of episode $j$, and at the start of each episode $k$, release a private version $\widetilde N_h^k(s,a)$ of the count $N_h^k(s,a):=\sum_{j=1}^{k-1}\sigma_h^j(s,a)$. 
Let us now discuss how private counts are computed. To this end, we let $N_h^{i,j}(s,a)=\sum_{k=i}^{j}\sigma_h^k(s,a)$ to denote a partial sum (\psum) of the counts in episodes $i$ through $j$, and consider a binary interval tree, each leaf node of which represents an episode (i.e., the tree has $k-1$ leaf nodes at the start of episode $k$). Each interior node of the tree represents the range of episodes covered by its children. At the start of episode $k$, first a noisy \psum\ corresponding to each node in the tree is released by perturbing it with an independent Laplace noise  $\text{Lap}(\frac{1}{\epsilon'})$, where $\epsilon' > 0$ is a given privacy parameter.\footnote{A random variable $X\sim\text{Lap}(b)$, with scale parameter $b >0$, if $\forall x\in\Real$, it's p.d.f. is given by $f_{X}(x)=\frac{1}{2b}\exp(-|x|/b)$.}
Then, the private count $\widetilde N_h^k(s,a)$ is computed by summing up the noisy \psums\ released by the set of nodes -- which has cardinality at most $O(\log k)$ -- that uniquely cover the range $[1,k-1]$.
Observe that, at the end of episode $k$, the counter only needs to store noisy \psums\ required for computing private counts at future episodes, and can safely discard \psums\ those are
no longer needed.

The counters corresponding to empirical rewards $C_h^k(s,a)$ and state transitions $N_h^k(s,a,s')$ follow the same underlying principle to release the respective private counts $\widetilde C_h^k(s,a)$ and $\widetilde N_h^k(s,a,s')$. The next lemma sums up the properties of the {\central}.


\begin{lemma}[Properties of {\central}]
\label{lem:central}
For any $\epsilon > 0$, {\central} with parameter $1/\epsilon'=\frac{3H\log K}{\epsilon}$ is $\epsilon$-DP. Furthermore, for any $\delta \in (0,1]$, it satisfies Assumption~\ref{ass:rand} with
$E_{\epsilon,\delta,1} = \frac{3H}{\epsilon}\sqrt{8\log^3K\log(6SAT/\delta)}$, and $E_{\epsilon,\delta,2} = \frac{3H}{\epsilon}\sqrt{8\log^3 K\log(6S^2AT/\delta)}$.
\end{lemma}
Lemma~\ref{lem:central} follows from the privacy guarantee of the Laplace mechanism, and the concentration bound on the sum of i.i.d. Laplace random variables \citep{dwork2014algorithmic}. Using Lemma~\ref{lem:central}, as corollaries of Theorem~\ref{thm:PO} and Theorem~\ref{thm:VI}, we obtain the regret and privacy guarantees for {\po} and {\vi} with the {\central}.

\begin{corollary}[Regret under JDP] 
\label{cor:JDP}
For any $\epsilon > 0$ and $\delta \in (0,1]$, both {\po} and {\vi}, if instantiated using {\central} with parameter $1/\epsilon'=\frac{3H\log K}{\epsilon}$, satisfy $\epsilon$-JDP. Furthermore, with probability at least $1-\delta$, we obtain the regret bounds:
\begin{align*}
    \mathcal{R}^{\text{{\po}}}(T) &= \widetilde O\left(\sqrt{S^2AH^3T} +  S^2AH^3/\epsilon\right) ,\\
    \mathcal{R}^{\text{{\vi}}}(T) &= \widetilde O\left(\sqrt{SAH^3T} +  S^2AH^3/\epsilon  \right) .
\end{align*}
\end{corollary}
We prove the JDP guarantees using the \emph{billboard model} \citep[Lemma 9]{hsu2016private} which, informally, states that an algorithm is JDP if the output sent to each user is a function of the user’s private data and a common quantity computed using a standard DP mechanism. Note that by Lemma~\ref{lem:central} and the post-processing property of DP \citep{dwork2014algorithmic}, the sequence of policies $(\pi^k)_{k}$ are $\epsilon$-DP. Therefore, by the billboard model, the actions $(a_h^k)_{h,k}$ suggested to all the users are $\epsilon$-JDP.

\begin{remark}
Corollary~\ref{cor:JDP}, to the best of our understanding, provides the first regret bound for private PO, and a correct regret bound for private VI as compared to \citet{vietri2020private}, under the requirement of JDP. 

\end{remark}

\subsection{Achieving LDP using {\local}}

The {\local}, at each episode $k$, release the private counts by injecting
Laplace noise into the aggregated statistics computed from the trajectory generated in that episode. Let us discuss how private counts for the number of visited states are computed.
At each episode $j$, given privacy parameter $\epsilon'>0$, {\local} perturbs $\sigma_h^j(s,a)$ with an independent Laplace noise $\text{Lap}(\frac{1}{\epsilon'})$, i.e. it makes $SAH$ noisy perturbations in total. The private counts for the $k$-th episode are computed as $\widetilde N_h^k(s,a)= \sum_{j=1}^{k-1}\widetilde \sigma_h^j(s,a)$, where $\widetilde \sigma_h^j(s,a)$ denotes the noisy perturbations. The private counts corresponding to empirical rewards $C_h^k(s,a)$ and state transitions $N_h^k(s,a,s')$ are computed similarly. The next lemma sums up the properties of the {\local}.



\begin{lemma}[Properties of {\local}]
\label{lem:local}
For any $\epsilon > 0$, {\local} with parameter $1/\epsilon'=\frac{3H}{\epsilon}$ is $\epsilon$-LDP. Furthermore, for any $\delta \in (0,1]$, it satisfies Assumption~\ref{ass:rand} with
$E_{\epsilon,\delta,1} = \frac{3H}{\epsilon}\sqrt{8K\log(6SAT/\delta)}$ and $E_{\epsilon,\delta,2} = \frac{3H}{\epsilon}\sqrt{8 K\log(6S^2AT/\delta)}$.
   \end{lemma}


\begin{corollary}[Regret under LDP] 
\label{cor:LDP}
For any $\epsilon > 0$ and $\delta \in (0,1]$, instantiating {\po} and {\vi} using {\local} with parameter $1/\epsilon'=\frac{3H}{\epsilon}$, we obtain, with probability  $\ge 1-\delta$, the regret bounds:
\begin{align*}
    \mathcal{R}^{\text{{\po}}}(T) &= \widetilde O\left(\sqrt{S^2AH^3T} +  S^2A\sqrt{H^5T}/\epsilon\right) ,\\
    \mathcal{R}^{\text{{\vi}}}(T) &= O\left(\sqrt{SAH^3T} +  S^2A\sqrt{H^5 T}/\epsilon \right).
\end{align*}
\end{corollary}
\begin{remark}
Corollary~\ref{cor:LDP}, to the best of our knowledge, provides the first regret guarantee for private PO, and an improved regret bound for private VI as compared to \citet{garcelon2020local},
under the requirement of LDP.
\end{remark}
\begin{remark}[JDP vs. LDP]
The noise level in the private counts is $O(\log k)$ under JDP and $O(k)$ under LDP. Due to this, the privacy cost for LDP is $\tilde O(\sqrt{T}/\epsilon)$, whereas for JDP it is only $\tilde O(1/\epsilon)$.
\end{remark}
\begin{remark}[Alternative LDP mechanisms]
Other than the Laplace noise, one can also use Bernoulli and Gaussian noise in the {\local} to achieve LDP \citep{kairouz2016discrete,wang2019locally}. Thanks to Theorem~\ref{thm:PO} and Theorem~\ref{thm:VI}, the regret bounds are readily obtained by plugging in the corresponding $E_{\epsilon,\delta,1}$ and $E_{\epsilon,\delta,2}$.
\end{remark}

\section{Experiments}
In this section, we conduct simple numerical experiments to verify our theoretical results for both policy-based and value-based algorithms\footnote{The code is available at \url{https://github.com/XingyuZhou989/PrivateTabularRL}}. 

\subsection{Settings}
Our experiment is based on the standard tabular MDP environment \texttt{RiverSwim}~\citep{strehl2008analysis,osband2013more}, illustrated in Fig.~\ref{fig:river}. It consists of six states and two actions `left' and `right', i.e., $S = 6$ and $A=2$. It starts with the left side and tries to reach the right side. At each step, if it chooses action `left', it will always succeed (cf. the dotted arrow). Otherwise it often fails (cf. the solid arrow). It only receives a small amount of reward if it reaches the leftmost side while obtaining a larger reward once it arrives at the rightmost state. Thus, this MDP naturally requires a sufficient exploration to obtain the optimal policy. Each episode is reset every $H=20$ steps.

\begin{figure}[h]
\centering
\includegraphics[width=6in]{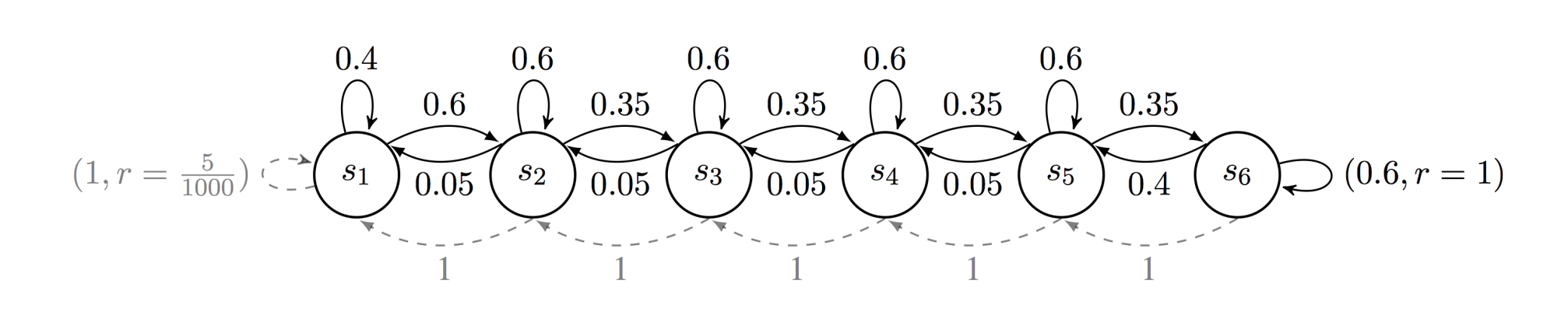}
\caption{\texttt{RiverSwim} MDP -- solid and dotted arrows denote the transitions under actions `right' and `left', respectively~\citep{osband2013more}. }
\label{fig:river}
\end{figure}

\subsection{Results}
We evaluate both {\po} and {\vi} under different privacy budget $\epsilon$ and also compare them with the corresponding non-private algorithms {\ucbvi}~\citep{azar2017minimax} and {\oppo} \citep{efroni2020optimistic}, respectively. 
We set all the parameters in our proposed algorithms as the same order as the theoretical results and tune the learning rate $\eta$ and the scaling of the confidence interval. We run $20$ independent experiments, each consisting of $K = 2 \cdot 10^4$ episodes. We plot the the average cumulative regret along with standard deviation for each setting, as shown in Fig.~\ref{fig:sim}

\begin{figure}[ht]
\centering
\begin{subfigure}{.45\textwidth}
  \centering
  \includegraphics[width=1\linewidth]{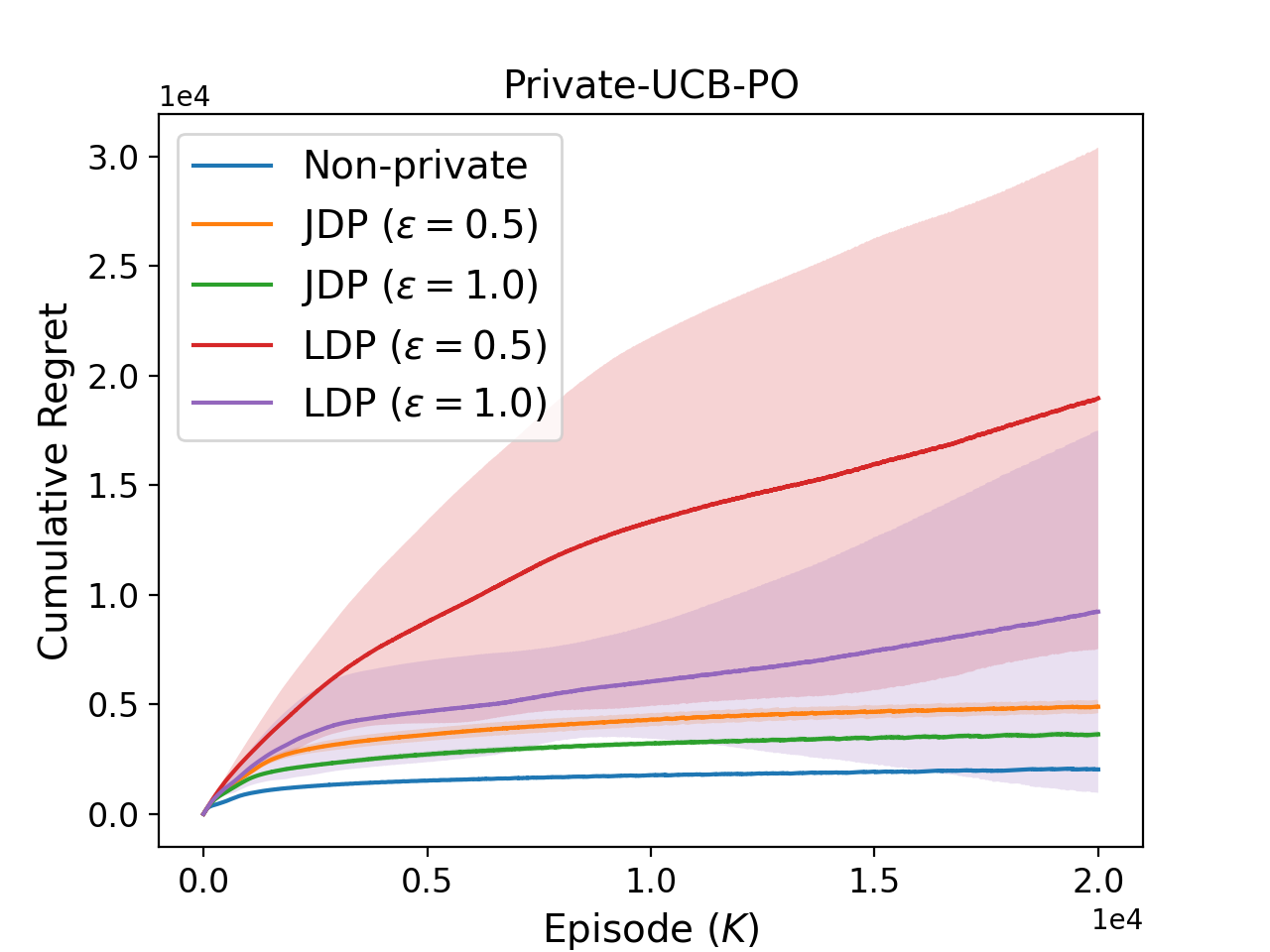}  
  \label{fig:sub-first}
\end{subfigure}
\begin{subfigure}{.45\textwidth}
  \centering
  \includegraphics[width=1\linewidth]{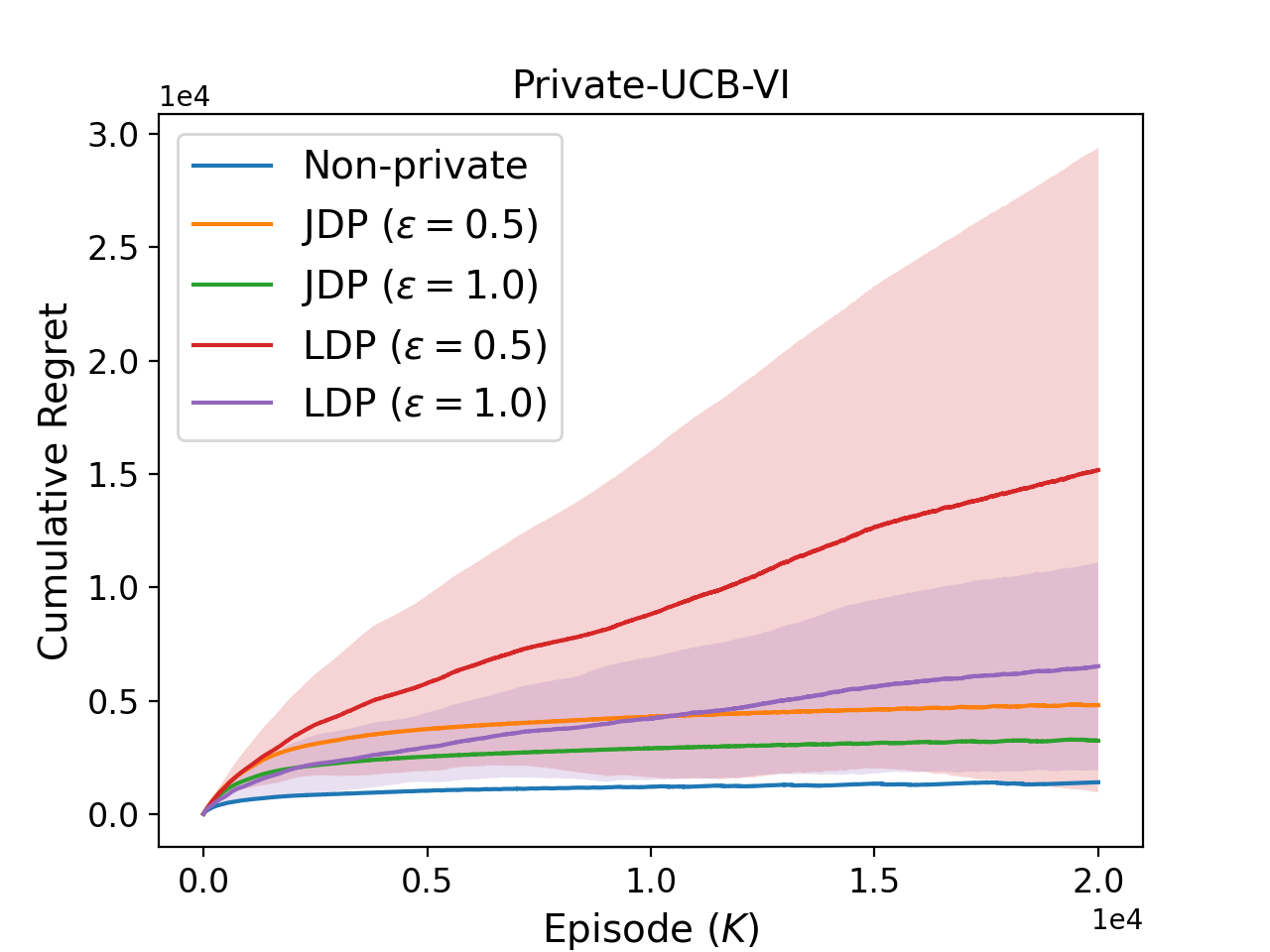}  
  \label{fig:sub-second}
\end{subfigure}
\caption{Cumulative regret vs. Episode under {\po} and  {\vi}}
\label{fig:sim}
\end{figure}

As suggested by theory, in both PO and VI cases,  we can see that the cost of privacy under the JDP requirement becomes negligible as number of episodes increases (since the cost is only a lower order additive term). But, under the stricter LDP requirement, the cost of privacy remains high (since the cost is multiplicative and is of the same order). Furthermore, it is worth noting that the cost of privacy increases with increasing protection level, i.e. with decreasing $\epsilon$.
\section{Conclusions}
In this work, we presented the first  private policy-optimization algorithm in tabular MDPs with regret guarantees under both JDP and LDP requirements. We also revisited private value-iteration algorithms by improving the regret bounds of existing results. These are achieved by developing a general framework for algorithm design and regret analysis in private tabular RL settings. Though we focus on statistical guarantees of private RL algorithms, it will be helpful to understand these from a practitioner's perspective. We leave this as a possible future direction. Another important direction is to apply our general framework to MDPs with function approximation, e.g., linear MDPs \citep{jin2019provably}, kernelized MDPs \citep{chowdhury2019online} and generic MDPs \citep{ayoub2020model}.

\bibliographystyle{unsrtnat}
\bibliography{main,2018library,bandit_RL}

\begin{thebibliography}{63}
\providecommand{\natexlab}[1]{#1}
\providecommand{\url}[1]{\texttt{#1}}
\expandafter\ifx\csname urlstyle\endcsname\relax
  \providecommand{\doi}[1]{doi: #1}\else
  \providecommand{\doi}{doi: \begingroup \urlstyle{rm}\Url}\fi

\bibitem[Gottesman et~al.(2019)Gottesman, Johansson, Komorowski, Faisal,
  Sontag, Doshi-Velez, and Celi]{gottesman2019guidelines}
Omer Gottesman, Fredrik Johansson, Matthieu Komorowski, Aldo Faisal, David
  Sontag, Finale Doshi-Velez, and Leo~Anthony Celi.
\newblock Guidelines for reinforcement learning in healthcare.
\newblock \emph{Nature medicine}, 25\penalty0 (1):\penalty0 16--18, 2019.

\bibitem[Li et~al.(2016)Li, Monroe, Ritter, Galley, Gao, and
  Jurafsky]{li2016deep}
Jiwei Li, Will Monroe, Alan Ritter, Michel Galley, Jianfeng Gao, and Dan
  Jurafsky.
\newblock Deep reinforcement learning for dialogue generation.
\newblock \emph{arXiv preprint arXiv:1606.01541}, 2016.

\bibitem[Gordon et~al.(2016)Gordon, Spaulding, Westlund, Lee, Plummer,
  Martinez, Das, and Breazeal]{gordon2016affective}
Goren Gordon, Samuel Spaulding, Jacqueline~Kory Westlund, Jin~Joo Lee, Luke
  Plummer, Marayna Martinez, Madhurima Das, and Cynthia Breazeal.
\newblock Affective personalization of a social robot tutor for children’s
  second language skills.
\newblock In \emph{Proceedings of the AAAI conference on artificial
  intelligence}, volume~30, 2016.

\bibitem[Li et~al.(2010)Li, Chu, Langford, and Schapire]{li2010contextual}
Lihong Li, Wei Chu, John Langford, and Robert~E Schapire.
\newblock A contextual-bandit approach to personalized news article
  recommendation.
\newblock In \emph{Proceedings of the 19th international conference on World
  wide web}, pages 661--670, 2010.

\bibitem[Vietri et~al.(2020)Vietri, Balle, Krishnamurthy, and
  Wu]{vietri2020private}
Giuseppe Vietri, Borja Balle, Akshay Krishnamurthy, and Steven Wu.
\newblock Private reinforcement learning with pac and regret guarantees.
\newblock In \emph{International Conference on Machine Learning}, pages
  9754--9764. PMLR, 2020.

\bibitem[Garcelon et~al.(2020)Garcelon, Perchet, Pike-Burke, and
  Pirotta]{garcelon2020local}
Evrard Garcelon, Vianney Perchet, Ciara Pike-Burke, and Matteo Pirotta.
\newblock Local differentially private regret minimization in reinforcement
  learning.
\newblock \emph{arXiv preprint arXiv:2010.07778}, 2020.

\bibitem[Dwork(2008)]{dwork2008differential}
Cynthia Dwork.
\newblock Differential privacy: A survey of results.
\newblock In \emph{International conference on theory and applications of
  models of computation}, pages 1--19. Springer, 2008.

\bibitem[Jain et~al.(2012)Jain, Kothari, and Thakurta]{jain2012differentially}
Prateek Jain, Pravesh Kothari, and Abhradeep Thakurta.
\newblock Differentially private online learning.
\newblock In \emph{Conference on Learning Theory}, pages 24--1. JMLR Workshop
  and Conference Proceedings, 2012.

\bibitem[Mishra and Thakurta(2015)]{mishra2015nearly}
Nikita Mishra and Abhradeep Thakurta.
\newblock (nearly) optimal differentially private stochastic multi-arm bandits.
\newblock In \emph{Proceedings of the Thirty-First Conference on Uncertainty in
  Artificial Intelligence}, pages 592--601, 2015.

\bibitem[Tossou and Dimitrakakis(2016)]{tossou2016algorithms}
Aristide Tossou and Christos Dimitrakakis.
\newblock Algorithms for differentially private multi-armed bandits.
\newblock In \emph{Proceedings of the AAAI Conference on Artificial
  Intelligence}, volume~30, 2016.

\bibitem[Shariff and Sheffet(2018)]{shariff2018differentially}
Roshan Shariff and Or~Sheffet.
\newblock Differentially private contextual linear bandits.
\newblock \emph{arXiv preprint arXiv:1810.00068}, 2018.

\bibitem[Dubey(2021)]{dubey2021no}
Abhimanyu Dubey.
\newblock No-regret algorithms for private gaussian process bandit
  optimization.
\newblock In \emph{International Conference on Artificial Intelligence and
  Statistics}, pages 2062--2070. PMLR, 2021.

\bibitem[Kearns et~al.(2014)Kearns, Pai, Roth, and Ullman]{kearns2014mechanism}
Michael Kearns, Mallesh Pai, Aaron Roth, and Jonathan Ullman.
\newblock Mechanism design in large games: Incentives and privacy.
\newblock In \emph{Proceedings of the 5th conference on Innovations in
  theoretical computer science}, pages 403--410, 2014.

\bibitem[Duchi et~al.(2013)Duchi, Jordan, and Wainwright]{duchi2013local}
John~C Duchi, Michael~I Jordan, and Martin~J Wainwright.
\newblock Local privacy and statistical minimax rates.
\newblock In \emph{2013 IEEE 54th Annual Symposium on Foundations of Computer
  Science}, pages 429--438. IEEE, 2013.

\bibitem[Ren et~al.(2020)Ren, Zhou, Liu, and Shroff]{ren2020multi}
Wenbo Ren, Xingyu Zhou, Jia Liu, and Ness~B Shroff.
\newblock Multi-armed bandits with local differential privacy.
\newblock \emph{arXiv preprint arXiv:2007.03121}, 2020.

\bibitem[Zheng et~al.(2020)Zheng, Cai, Huang, Li, and Wang]{zheng2020locally}
Kai Zheng, Tianle Cai, Weiran Huang, Zhenguo Li, and Liwei Wang.
\newblock Locally differentially private (contextual) bandits learning.
\newblock \emph{arXiv preprint arXiv:2006.00701}, 2020.

\bibitem[Zhou and Tan(2020)]{zhou2020local}
Xingyu Zhou and Jian Tan.
\newblock Local differential privacy for bayesian optimization.
\newblock \emph{arXiv preprint arXiv:2010.06709}, 2020.

\bibitem[Chowdhury et~al.(2021)Chowdhury, Zhou, and
  Shroff]{chowdhury2021adaptive}
Sayak~Ray Chowdhury, Xingyu Zhou, and Ness Shroff.
\newblock Adaptive control of differentially private linear quadratic systems.
\newblock In \emph{2021 IEEE International Symposium on Information Theory
  (ISIT)}, pages 485--490. IEEE, 2021.

\bibitem[Silver et~al.(2017)Silver, Schrittwieser, Simonyan, Antonoglou, Huang,
  Guez, Hubert, Baker, Lai, Bolton, et~al.]{silver2017mastering}
David Silver, Julian Schrittwieser, Karen Simonyan, Ioannis Antonoglou, Aja
  Huang, Arthur Guez, Thomas Hubert, Lucas Baker, Matthew Lai, Adrian Bolton,
  et~al.
\newblock Mastering the game of go without human knowledge.
\newblock \emph{nature}, 550\penalty0 (7676):\penalty0 354--359, 2017.

\bibitem[Duan et~al.(2016)Duan, Chen, Houthooft, Schulman, and
  Abbeel]{duan2016benchmarking}
Yan Duan, Xi~Chen, Rein Houthooft, John Schulman, and Pieter Abbeel.
\newblock Benchmarking deep reinforcement learning for continuous control.
\newblock In \emph{International conference on machine learning}, pages
  1329--1338. PMLR, 2016.

\bibitem[Wang et~al.(2018)Wang, Li, and He]{wang2018deep}
William~Yang Wang, Jiwei Li, and Xiaodong He.
\newblock Deep reinforcement learning for nlp.
\newblock In \emph{Proceedings of the 56th Annual Meeting of the Association
  for Computational Linguistics: Tutorial Abstracts}, pages 19--21, 2018.

\bibitem[Williams(1992)]{williams1992simple}
Ronald~J Williams.
\newblock Simple statistical gradient-following algorithms for connectionist
  reinforcement learning.
\newblock \emph{Machine learning}, 8\penalty0 (3-4):\penalty0 229--256, 1992.

\bibitem[Kakade(2001)]{kakade2001natural}
Sham~M Kakade.
\newblock A natural policy gradient.
\newblock \emph{Advances in neural information processing systems}, 14, 2001.

\bibitem[Schulman et~al.(2015)Schulman, Levine, Abbeel, Jordan, and
  Moritz]{schulman2015trust}
John Schulman, Sergey Levine, Pieter Abbeel, Michael Jordan, and Philipp
  Moritz.
\newblock Trust region policy optimization.
\newblock In \emph{International conference on machine learning}, pages
  1889--1897. PMLR, 2015.

\bibitem[Schulman et~al.(2017)Schulman, Wolski, Dhariwal, Radford, and
  Klimov]{schulman2017proximal}
John Schulman, Filip Wolski, Prafulla Dhariwal, Alec Radford, and Oleg Klimov.
\newblock Proximal policy optimization algorithms.
\newblock \emph{arXiv preprint arXiv:1707.06347}, 2017.

\bibitem[Konda and Tsitsiklis(2000)]{konda2000actor}
Vijay~R Konda and John~N Tsitsiklis.
\newblock Actor-critic algorithms.
\newblock In \emph{Advances in neural information processing systems}, pages
  1008--1014. Citeseer, 2000.

\bibitem[Liu et~al.(2019)Liu, Cai, Yang, and Wang]{liu2019neural}
Boyi Liu, Qi~Cai, Zhuoran Yang, and Zhaoran Wang.
\newblock Neural proximal/trust region policy optimization attains globally
  optimal policy.
\newblock \emph{arXiv preprint arXiv:1906.10306}, 2019.

\bibitem[Wang et~al.(2019{\natexlab{a}})Wang, Cai, Yang, and
  Wang]{wang2019neural}
Lingxiao Wang, Qi~Cai, Zhuoran Yang, and Zhaoran Wang.
\newblock Neural policy gradient methods: Global optimality and rates of
  convergence.
\newblock \emph{arXiv preprint arXiv:1909.01150}, 2019{\natexlab{a}}.

\bibitem[Cai et~al.(2020)Cai, Yang, Jin, and Wang]{cai2020provably}
Qi~Cai, Zhuoran Yang, Chi Jin, and Zhaoran Wang.
\newblock Provably efficient exploration in policy optimization.
\newblock In \emph{International Conference on Machine Learning}, pages
  1283--1294. PMLR, 2020.

\bibitem[Efroni et~al.(2020{\natexlab{a}})Efroni, Shani, Rosenberg, and
  Mannor]{efroni2020optimistic}
Yonathan Efroni, Lior Shani, Aviv Rosenberg, and Shie Mannor.
\newblock Optimistic policy optimization with bandit feedback.
\newblock \emph{arXiv preprint arXiv:2002.08243}, 2020{\natexlab{a}}.

\bibitem[Guha~Thakurta and Smith(2013)]{guha2013nearly}
Abhradeep Guha~Thakurta and Adam Smith.
\newblock (nearly) optimal algorithms for private online learning in
  full-information and bandit settings.
\newblock \emph{Advances in Neural Information Processing Systems},
  26:\penalty0 2733--2741, 2013.

\bibitem[Agarwal and Singh(2017)]{agarwal2017price}
Naman Agarwal and Karan Singh.
\newblock The price of differential privacy for online learning.
\newblock In \emph{International Conference on Machine Learning}, pages 32--40.
  PMLR, 2017.

\bibitem[Tossou and Dimitrakakis(2017)]{tossou2017achieving}
Aristide Tossou and Christos Dimitrakakis.
\newblock Achieving privacy in the adversarial multi-armed bandit.
\newblock In \emph{Proceedings of the AAAI Conference on Artificial
  Intelligence}, volume~31, 2017.

\bibitem[Hu et~al.(2021)Hu, Huang, and Mehta]{hu2021optimal}
Bingshan Hu, Zhiming Huang, and Nishant~A Mehta.
\newblock Optimal algorithms for private online learning in a stochastic
  environment.
\newblock \emph{arXiv preprint arXiv:2102.07929}, 2021.

\bibitem[Sajed and Sheffet(2019)]{sajed2019optimal}
Touqir Sajed and Or~Sheffet.
\newblock An optimal private stochastic-mab algorithm based on optimal private
  stopping rule.
\newblock In \emph{International Conference on Machine Learning}, pages
  5579--5588. PMLR, 2019.

\bibitem[Gajane et~al.(2018)Gajane, Urvoy, and Kaufmann]{gajane2018corrupt}
Pratik Gajane, Tanguy Urvoy, and Emilie Kaufmann.
\newblock Corrupt bandits for preserving local privacy.
\newblock In \emph{Algorithmic Learning Theory}, pages 387--412. PMLR, 2018.

\bibitem[Chen et~al.(2020)Chen, Zheng, Zhou, Yang, Chen, and
  Wang]{chen2020locally}
Xiaoyu Chen, Kai Zheng, Zixin Zhou, Yunchang Yang, Wei Chen, and Liwei Wang.
\newblock (locally) differentially private combinatorial semi-bandits.
\newblock In \emph{International Conference on Machine Learning}, pages
  1757--1767. PMLR, 2020.

\bibitem[Balle et~al.(2016)Balle, Gomrokchi, and
  Precup]{balle2016differentially}
Borja Balle, Maziar Gomrokchi, and Doina Precup.
\newblock Differentially private policy evaluation.
\newblock In \emph{International Conference on Machine Learning}, pages
  2130--2138. PMLR, 2016.

\bibitem[Wang and Hegde(2019)]{wang2019privacy}
Baoxiang Wang and Nidhi Hegde.
\newblock Privacy-preserving q-learning with functional noise in continuous
  state spaces.
\newblock \emph{arXiv preprint arXiv:1901.10634}, 2019.

\bibitem[Ono and Takahashi(2020)]{ono2020locally}
Hajime Ono and Tsubasa Takahashi.
\newblock Locally private distributed reinforcement learning.
\newblock \emph{arXiv preprint arXiv:2001.11718}, 2020.

\bibitem[Puterman(1994)]{puterman94}
Martin~L. Puterman.
\newblock \emph{{M}arkov Decision Processes --- Discrete Stochastic Dynamic
  Programming}.
\newblock John Wiley \& Sons, Inc., New York, NY, 1994.

\bibitem[Sutton(1988)]{sutton1988learning}
Richard~S Sutton.
\newblock Learning to predict by the methods of temporal differences.
\newblock \emph{Machine learning}, 3\penalty0 (1):\penalty0 9--44, 1988.

\bibitem[Dwork et~al.(2014)Dwork, Roth, et~al.]{dwork2014algorithmic}
Cynthia Dwork, Aaron Roth, et~al.
\newblock The algorithmic foundations of differential privacy.
\newblock 2014.

\bibitem[Hsu et~al.(2016)Hsu, Huang, Roth, Roughgarden, and Wu]{hsu2016private}
Justin Hsu, Zhiyi Huang, Aaron Roth, Tim Roughgarden, and Zhiwei~Steven Wu.
\newblock Private matchings and allocations.
\newblock \emph{SIAM Journal on Computing}, 45\penalty0 (6):\penalty0
  1953--1984, 2016.

\bibitem[Auer et~al.(2002)Auer, Cesa-Bianchi, and Fischer]{Auer02}
P.~Auer, N.~Cesa-Bianchi, and P.~Fischer.
\newblock Finite-time analysis of the multiarmed bandit problem.
\newblock \emph{Machine Learning}, 47\penalty0 (2-3):\penalty0 235--256, 2002.

\bibitem[Jaksch et~al.(2010)Jaksch, Ortner, and Auer]{jaksch2010near}
Thomas Jaksch, Ronald Ortner, and Peter Auer.
\newblock Near-optimal regret bounds for reinforcement learning.
\newblock \emph{Journal of Machine Learning Research}, 11\penalty0
  (Apr):\penalty0 1563--1600, 2010.

\bibitem[Beck and Teboulle(2003)]{beck2003mirror}
Amir Beck and Marc Teboulle.
\newblock Mirror descent and nonlinear projected subgradient methods for convex
  optimization.
\newblock \emph{Operations Research Letters}, 31\penalty0 (3):\penalty0
  167--175, 2003.

\bibitem[Azar et~al.(2017)Azar, Osband, and Munos]{azar2017minimax}
Mohammad~Gheshlaghi Azar, Ian Osband, and R{\'e}mi Munos.
\newblock Minimax regret bounds for reinforcement learning.
\newblock In \emph{International Conference on Machine Learning}, pages
  263--272. PMLR, 2017.

\bibitem[Dann et~al.(2017)Dann, Lattimore, and Brunskill]{dann2017unifying}
Christoph Dann, Tor Lattimore, and Emma Brunskill.
\newblock Unifying pac and regret: Uniform pac bounds for episodic
  reinforcement learning.
\newblock \emph{arXiv preprint arXiv:1703.07710}, 2017.

\bibitem[Efroni et~al.(2019)Efroni, Merlis, Ghavamzadeh, and
  Mannor]{efroni2019tight}
Yonathan Efroni, Nadav Merlis, Mohammad Ghavamzadeh, and Shie Mannor.
\newblock Tight regret bounds for model-based reinforcement learning with
  greedy policies.
\newblock \emph{arXiv preprint arXiv:1905.11527}, 2019.

\bibitem[Zanette and Brunskill(2019)]{zanette2019tighter}
Andrea Zanette and Emma Brunskill.
\newblock Tighter problem-dependent regret bounds in reinforcement learning
  without domain knowledge using value function bounds.
\newblock In \emph{International Conference on Machine Learning}, pages
  7304--7312. PMLR, 2019.

\bibitem[Efroni et~al.(2020{\natexlab{b}})Efroni, Mannor, and
  Pirotta]{efroni2020exploration}
Yonathan Efroni, Shie Mannor, and Matteo Pirotta.
\newblock Exploration-exploitation in constrained mdps.
\newblock \emph{arXiv preprint arXiv:2003.02189}, 2020{\natexlab{b}}.

\bibitem[Chan et~al.(2010)Chan, Shi, and Song]{chan2010private}
TH~Hubert Chan, Elaine Shi, and Dawn Song.
\newblock Private and continual release of statistics.
\newblock In \emph{International Colloquium on Automata, Languages, and
  Programming}, pages 405--417. Springer, 2010.

\bibitem[Kairouz et~al.(2016)Kairouz, Bonawitz, and
  Ramage]{kairouz2016discrete}
Peter Kairouz, Keith Bonawitz, and Daniel Ramage.
\newblock Discrete distribution estimation under local privacy.
\newblock In \emph{International Conference on Machine Learning}, pages
  2436--2444. PMLR, 2016.

\bibitem[Wang et~al.(2019{\natexlab{b}})Wang, Zhao, Yang, and
  Ren]{wang2019locally}
Teng Wang, Jun Zhao, Xinyu Yang, and Xuebin Ren.
\newblock Locally differentially private data collection and analysis.
\newblock \emph{arXiv preprint arXiv:1906.01777}, 2019{\natexlab{b}}.

\bibitem[Strehl and Littman(2008)]{strehl2008analysis}
Alexander~L Strehl and Michael~L Littman.
\newblock An analysis of model-based interval estimation for \textsc{M}arkov
  decision processes.
\newblock \emph{Journal of Computer and System Sciences}, 74\penalty0
  (8):\penalty0 1309--1331, 2008.

\bibitem[Osband et~al.(2013)Osband, Russo, and Van~Roy]{osband2013more}
Ian Osband, Dan Russo, and Benjamin Van~Roy.
\newblock (more) efficient reinforcement learning via posterior sampling.
\newblock In \emph{Advances in Neural Information Processing Systems 26
  (NIPS)}, pages 3003--3011, 2013.

\bibitem[Jin et~al.(2019)Jin, Yang, Wang, and Jordan]{jin2019provably}
Chi Jin, Zhuoran Yang, Zhaoran Wang, and Michael~I Jordan.
\newblock Provably efficient reinforcement learning with linear function
  approximation.
\newblock \emph{arXiv preprint arXiv:1907.05388}, 2019.

\bibitem[Chowdhury and Gopalan(2019)]{chowdhury2019online}
Sayak~Ray Chowdhury and Aditya Gopalan.
\newblock Online learning in kernelized markov decision processes.
\newblock In \emph{The 22nd International Conference on Artificial Intelligence
  and Statistics}, pages 3197--3205. PMLR, 2019.

\bibitem[Ayoub et~al.(2020)Ayoub, Jia, Szepesvari, Wang, and
  Yang]{ayoub2020model}
Alex Ayoub, Zeyu Jia, Csaba Szepesvari, Mengdi Wang, and Lin~F Yang.
\newblock Model-based reinforcement learning with value-targeted regression.
\newblock \emph{arXiv preprint arXiv:2006.01107}, 2020.

\bibitem[Weissman et~al.(2003)Weissman, Ordentlich, Seroussi, Verdu, and
  Weinberger]{weissman2003inequalities}
Tsachy Weissman, Erik Ordentlich, Gadiel Seroussi, Sergio Verdu, and Marcelo~J
  Weinberger.
\newblock Inequalities for the \textsc{L1} deviation of the empirical
  distribution.
\newblock \emph{Hewlett-Packard Labs, Technical Report}, 2003.

\bibitem[Jin et~al.(2020)Jin, Jin, Luo, Sra, and Yu]{jin2020learning}
Chi Jin, Tiancheng Jin, Haipeng Luo, Suvrit Sra, and Tiancheng Yu.
\newblock Learning adversarial markov decision processes with bandit feedback
  and unknown transition.
\newblock In \emph{International Conference on Machine Learning}, pages
  4860--4869. PMLR, 2020.

\bibitem[Maurer and Pontil(2009)]{maurer2009empirical}
A~Maurer and M~Pontil.
\newblock Empirical {B}ernstein bounds and sample variance penalization.
\newblock In \emph{22nd Conference on Learning Theory (COLT)}, 2009.

\end{thebibliography}

\begin{appendix}
\onecolumn
\begin{center}
    {\huge Appendix}
\end{center}

\section{Proofs for Section~\ref{sec:po} }

We first define the \emph{non-private} mean empirical costs and empirical transition probabilities as follows. 
\begin{align*}
 \bar{c}_h^k(s,a)  := \frac{{C}_h^k(s,a)}{{{N}_h^k(s,a)}\vee 1 }, \quad \bar{P}_h^k(s'|s,a):= \frac{{N}_h^k(s,a,s')}{ {{N}_h^k(s,a)}\vee 1}~.
\end{align*}
Based on them, we define the following events
\begin{align*}
    &F^c = \left\{ \exists s,a,h,k: |c_h(s,a) - \bar{c}_h^k(s,a) | \ge \sqrt{\frac{ 2\ln\frac{4SAT}{\delta'}}{ {N_h^k(s,a) \vee 1} } } \right\}\\
    &F^p = \left\{ \exists s,a,h,k: \norm{P_h(\cdot|s,a) - \bar{P}_h^k(\cdot|s,a)}_1  \ge \sqrt{\frac{ 4S\ln\frac{6SAT}{\delta'}}{ {N_h^k(s,a) \vee 1} } } \right\}
\end{align*}
and $\bar{G} := F^c \cup F^p$, which is the the complement of the good event $G$. The next lemma states that the good event happens with high probability. 


\begin{lemma}
For any $\delta \in (0,1]$, $\prob{G} \ge 1-\delta$. 
\end{lemma}
\begin{proof}
We first have $\prob{F^c} \le \delta/2$. This follows from Hoeffding’s inequality and union bound over all $s,a$ and all possible values of $N_h^k(s, a)$ and $k$. Note that when $N_h^k(s,a) = 0$, this bound holds trivially. We also have $\prob{F^p} \le \delta/2$. This holds by~\cite[Theorem 2.1]{weissman2003inequalities} along with the application of union bound over all $s,a$ and all possible values of $N_h^k(s, a)$ and $k$. Note that when $N_h^k(s,a) = 0$, this bound holds trivially.
\end{proof}

Now, we are ready to present the proof for Lemma~\ref{lem:conc_po}.
\begin{proof}[Proof of Lemma~\ref{lem:conc_po}]
Assume that both the good event $G$ and the event in Assumption~\ref{ass:rand} hold. We first study the concentration of the private cost estimate. Note that under the event in  Assumption~\ref{ass:rand}, 
 \begin{align*}
    \left|\frac{\widetilde{C}_h^k(s,a)}{(\widetilde{N}_h^k(s,a) +  E_{\epsilon,\delta,1}) \vee 1 } - \frac{{C}_h^k(s,a)}{(\widetilde{N}_h^k(s,a) +  E_{\epsilon,\delta,1}) \vee 1} \right| \le \frac{E_{\epsilon,\delta,1}}{(\widetilde{N}_h^k(s,a) +  E_{\epsilon,\delta,1}) \vee 1},
\end{align*}
since $\widetilde{N}_h^k(s,a) +  E_{\epsilon,1} \ge N_h^k(s,a) \ge 0$ and $|\widetilde{C}_h^k(s,a) - {C}_h^k(s,a) | \le E_{\epsilon,\delta,1}$. Moreover, we have 
    \begin{align*}
        &\left|\frac{{C}_h^k(s,a)}{(\widetilde{N}_h^k(s,a) +  E_{\epsilon,\delta,1})\vee 1} - c_h(s,a) \right| \\
        \le &\left| c_h(s,a) \left(\frac{N_h^k(s,a) \vee 1}{(\widetilde{N}_h^k(s,a) +  E_{\epsilon,\delta,1})\vee 1} - 1\right)\right| + \left|\frac{N_h^k(s,a)\vee 1}{(\widetilde{N}_h^k(s,a) + E_{\epsilon,\delta,1}) \vee 1} \left(\frac{C_h^k(s,a)}{N_h^k(s,a)\vee 1} - c_h(s,a)\right) \right|\\
        \lep{a} &c_h(s,a) \left|1-\frac{N_h^k(s,a)\vee 1}{(\widetilde{N}_h^k(s,a) +  E_{\epsilon,\delta,1}) \vee 1}  \right| + \frac{N_h^k(s,a)\vee 1}{(\widetilde{N}_h^k(s,a) +  E_{\epsilon,\delta,1})\vee 1} \frac{L_c(\delta)}{\sqrt{N_h^k(s,a)\vee 1}}\\
        \le & \frac{2E_{\epsilon,\delta,1} }{(\widetilde{N}_h^k(s,a) + E_{\epsilon,\delta,1})\vee 1} + \frac{L_c(\delta)\sqrt{N_h^k(s,a)\vee 1}  }{(\widetilde{N}_h^k(s,a) + E_{\epsilon,\delta,1}) \vee 1},
    \end{align*}
    where (a) holds by the concentration of the cost estimate under good event $G$ with $L_c(\delta) := \sqrt{2\ln\frac{4SAT}{\delta}}$.
    Furthermore, we have 
    \begin{align*}
        \frac{L_c(\delta)\sqrt{N_h^k(s,a)\vee 1}  }{(\widetilde{N}_h^k(s,a) +  E_{\epsilon,\delta,1})\vee 1} \le \frac{L_c(\delta)\sqrt{ (\widetilde{N}_h^k(s,a) +  E_{\epsilon,\delta,1}) \vee 1}  }{(\widetilde{N}_h^k(s,a) +  E_{\epsilon,\delta,1})\vee 1} = \frac{L_c(\delta)}{\sqrt{(\widetilde{N}_h^k(s,a) +   E_{\epsilon,\delta,1})\vee 1}}.
    \end{align*}
    Putting everything together, yields
    \begin{align*}
        |c_h(s,a) - \widetilde{c}_h^k(s,a)| \le \frac{3E_{\epsilon,\delta,1} }{(\widetilde{N}_h^k(s,a) +  E_{\epsilon,\delta,1})\vee 1} + \frac{L_c(\delta)}{\sqrt{(\widetilde{N}_h^k(s,a) +  E_{\epsilon,\delta,1}) \vee 1}}.
    \end{align*}
    Now, we turn to bound the transition dynamics. The error between the true transition probability and the private estimate can be decomposed as 
    \begin{align*}
        &\sum_{s'} |P_h(s'|,s,a) - \widetilde{P}_h^k(s'|s,a)|\\
        =&\sum_{s'} \left| \frac{\widetilde{N}_h^k(s,a,s')}{(\widetilde{N}_h^k(s,a) +  E_{\epsilon,\delta,1})\vee 1} - P_h(s'|s,a)\right|\\
        \le& \underbrace{\sum_{s'}\left|  \frac{{N}_h^k(s,a,s')}{(\widetilde{N}_h^k(s,a) +  E_{\epsilon,\delta,1})\vee 1} - P_h(s'|s,a) \right|}_{\mathcal{P}_1} +  \underbrace{\sum_{s'}\left|  \frac{\widetilde{N}_h^k(s,a,s') - {N}_h^k(s,a,s')}{(\widetilde{N}_h^k(s,a) +  E_{\epsilon,\delta,1})\vee 1} \right|}_{\mathcal{P}_2}.
    \end{align*}
        For $\mathcal{P}_1$, we have 
    \begin{align*}
        &\mathcal{P}_1=\sum_{s'}\left|\frac{N_h^k(s,a,s')}{N_h^k(s,a)\vee 1}\frac{N_h^k(s,a)\vee 1}{(\widetilde{N}_h^k(s,a) +  E_{\epsilon,\delta,1})\vee 1} - P_h(s'|s,a)\right|\\
        =&\sum_{s'}\left|\left( \frac{N_h^k(s,a,s')}{N_h^k(s,a)\vee 1} - P_h(s'|s,a) \right) \frac{N_h^k(s,a)\vee 1}{(\widetilde{N}_h^k(s,a) +  E_{\epsilon,\delta,1})\vee 1} + P_h(s'|s,a)\left(\frac{N_h^k(s,a)\vee 1}{(\widetilde{N}_h^k(s,a) +  E_{\epsilon,\delta,1})\vee 1}-1 \right) \right|\\
        \le& \frac{N_h^k(s,a)\vee 1}{(\widetilde{N}_h^k(s,a) +  E_{\epsilon,\delta,1})\vee 1} \norm{\bar{P}_h^k(\cdot|s,a) - P_h(\cdot|s,a)}_1 + \sum_{s'}\left( P_h(s'|s,a) \frac{2 E_{\epsilon,\delta,1}}{(\widetilde{N}_h^k(s,a) +  E_{\epsilon,\delta,1})\vee 1}\right)\\
        \lep{a}&\frac{N_h^k(s,a)\vee 1}{(\widetilde{N}_h^k(s,a) +  E_{\epsilon,\delta,1})\vee 1} \frac{L_p(\delta)}{\sqrt{N_h^k(s,a)\vee 1}} + \frac{2 E_{\epsilon,\delta,1}}{(\widetilde{N}_h^k(s,a) + E_{\epsilon,\delta,1})\vee 1}\\
        \le & \frac{L_p(\delta)}{\sqrt{(\widetilde{N}_h^k(s,a) +  E_{\epsilon,\delta,1} )\vee 1}} + \frac{2 E_{\epsilon,\delta,1}}{(\widetilde{N}_h^k(s,a) +  E_{\epsilon,\delta,1})\vee 1},
    \end{align*}
   where (a) holds by concentration of transition probability under good event $G$ with $L_p(\delta):=\sqrt{4S\ln\frac{6SAT}{\delta'}}$.
    For $\mathcal{P}_2$, we have
    \begin{align*}
        \mathcal{P}_2 \le \sum_{s'}\frac{|E_{\epsilon,\delta,2}|}{(\widetilde{N}_h^k(s,a) +  E_{\epsilon,\delta,1})\vee 1} = \frac{S E_{\epsilon,\delta,2}}{(\widetilde{N}_h^k(s,a) +  E_{\epsilon,\delta,1})\vee 1}.
    \end{align*}
    Putting together $\mathcal{P}_1$ and $\mathcal{P}_2$, yields
    \begin{align*}
       \norm{P_h(\cdot|s,a) - \widetilde{P}_h^k(\cdot|s,a)}_1 \le  \frac{L_p(\delta)}{\sqrt{(\widetilde{N}_h^k(s,a) +  E_{\epsilon,1} )\vee 1}}+ \frac{SE_{\epsilon,\delta,2}+2E_{\epsilon,\delta,1}}{(\widetilde{N}_h^k(s,a) +  E_{\epsilon,\delta,1})\vee 1}.
    \end{align*}
    Finally, applying union bound over good event $G$ and the event in Assumption~\ref{ass:rand}, yields the required result in Lemma~\ref{lem:conc_po}. 
\end{proof}

We turn to present the proof for Theorem~\ref{thm:PO} as follows.
\begin{proof}[Proof of Theorem~\ref{thm:PO}]
As in the non-private case, we first decompose the regret by using the extended value difference lemma~\cite[Lemma 1]{efroni2020optimistic}.
\begin{align*}
    \cR(T)&=\sum\nolimits_{k=1}^K \left(V_1^{\pi^k}(s_1^k) - V_1^{\pi^*}(s_1^k)\right) = \sum\nolimits_{k=1}^k \left(V_1^{\pi^k}(s_1^k) -\widetilde{V}_1^k(s_1^k) + \widetilde{V}_1^k(s_1^k) - V_1^{\pi^*}(s_1^k)\right)\\
    &= \underbrace{\sum\nolimits_{k=1}^k \left(V_1^{\pi^k}(s_1^k)-\widetilde{V}_1^k(s_1^k)\right)}_{\mathcal{T}_1}  +\underbrace{\sum\nolimits_{k=1}^K\sum\nolimits_{h=1}^H \ex{ \inner{\widetilde{Q}_h^k(s_h,\cdot)}{\pi_h^k(\cdot|s_h) - \pi_h^*(\cdot|s_h)}| s_1^k,\pi^*} }_{\mathcal{T}_2}\\
    &\quad\quad+\underbrace{\sum\nolimits_{k=1}^K\sum\nolimits_{h=1}^H \ex{ \widetilde{Q}_h^k(s_h,a_h) - c_h(s_h,a_h) - P_h(\cdot|s_h,a_h)\widetilde{V}_{h+1}^k|s_1^k, \pi^*} }_{\mathcal{T}_3}.
\end{align*}
We then need to bound each of the three terms. 

\textbf{Analysis of $\mathcal{T}_2$}. To start with, we can bound $\mathcal{T}_2$ by following standard mirror descent analysis under KL divergence. Specifically, by~\cite[Lemma 17]{efroni2020optimistic}, we have for any $h \in [H]$, $s \in \mathcal{S}$ and any policy $\pi$
\begin{align*}
    \sum_{k=1}^K \inner{\widetilde{Q}_h^k(s,\cdot)}{\pi_h^k(\cdot|s) - \pi_h(\cdot|s)} &\le \frac{\log A}{\eta} + \frac{\eta}{2}\sum_{k=1}^K\sum_a \pi_h^k(a|s) (Q_h^k(s,a))^2\\
    &\lep{a}  \frac{\log A}{\eta} + \frac{\eta H^2 K}{2},
\end{align*}
where (a) holds by  $Q_h^k(s,a) \in [0,H]$, which follows from the truncated update of $Q$-value in Algorithm~\ref{alg:PO} (line 9). Thus, we can bound $\mathcal{T}_2$ as follows.
\begin{align*}
    \mathcal{T}_2 &= \sum_{k=1}^K \sum_{h=1}^H \ex{\inner{\widetilde{Q}_h^k(s_h^k,\cdot)}{\pi_h^k(\cdot|s_h^k) - \pi_h^*(\cdot|s_h^k)}|s_1^k, \pi^* } \le  \frac{H\log A}{\eta} + \frac{\eta H^3 K}{2}.
\end{align*}
Choosing $\eta = \sqrt{2\log A/(H^2K)}$, yields
\begin{align}
    \mathcal{T}_2 \le \sqrt{2H^4 K \log A}.
\end{align}

\textbf{Analysis of $\mathcal{T}_3$}. First, by the update rule of $Q$-value in Algorithm~\ref{alg:PO} and  $P_h(\cdot|s,a)V_{h+1}:=\sum_{s'} P_h(s'|s,a) V_{h+1}(s')$, we have 
\begin{align*}
    \widetilde{Q}_h^k(s,a) &= \min\{H, \max\{0,\widetilde{c}_h^k(s,a)  + \sum_{s'\in \mathcal{S}} \widetilde{V}_{h+1}^k(s')\widetilde{P}_h^k(s'|s,a) - \beta_h^{k}(s,a)\} \}\\
    &= \min\{H, \max\{0,\widetilde{c}_h^k(s,a)  +  \widetilde{P}_h(\cdot|s_h^k,a_h^k)\widetilde{V}_{h+1}^k - \beta_h^{k}(s,a)\} \}\\
    &\le \max\left\{0,{\widetilde{c}_h^k(s,a)}-\beta_h^{k,c}(s,a) +  \widetilde{P}_h(\cdot|s_h^k,a_h^k)\widetilde{V}_{h+1}^k - H\beta_h^{k,p}(s,a) \right\}\nonumber\\
    &\lep{a} \max\left\{0, {\widetilde{c}_h^k(s,a)}-\beta_h^{k,c}(s,a) \right\} + \max\left\{0, \widetilde{P}_h(\cdot|s_h^k,a_h^k)\widetilde{V}_{h+1}^k- \beta_h^{k,pv}(s,a)   \right\}
\end{align*}
where (a) holds since for any $a,b$, $\max\{a+b,0\} \le \max\{a,0\} + \max\{b,0\}$. Thus, for any $(k,h,s,a)$, we have 
\begin{align}
    &\widetilde{Q}_h^k(s,a) - c_h(s,a) - P_h(\cdot|s,a)\widetilde{V}_{h+1}^k\nonumber\\
    \le &\max\left\{0, {\widetilde{c}_h^k(s,a)}-\beta_h^{k,c}(s,a) \right\} + \max\left\{0,\widetilde{P}_h^{k}(\cdot|s,a)\widetilde{V}_h^k - H\beta_h^{k,p}(s,a)   \right\} - c_h(s,a) - P_h(\cdot|s,a)\widetilde{V}_{h+1}^k\nonumber\\
    =&\max\left\{ -c_h(s,a), {\widetilde{c}_h^k(s,a)}- c_h(s,a)-\beta_h^{k,c}(s,a)\right\}\nonumber\\
    &+ \max\left\{ - P_h(\cdot|s,a)\widetilde{V}_{h+1}^k, \widetilde{P}_h^{k}(\cdot|s,a)\widetilde{V}_h^k - P_h(\cdot|s,a)\widetilde{V}_{h+1}^k- H\beta_h^{k,p}(s,a) \right\}\nonumber\\
    \le& \max\left\{ 0, {\widetilde{c}_h^k(s,a)}- c_h(s,a)-\beta_h^{k,c}(s,a)\right\} \label{eq:c_t3}\\
    &+ \max\left\{0, \widetilde{P}_h^{k}(\cdot|s,a)\widetilde{V}_h^k - P_h(\cdot|s,a)\widetilde{V}_{h+1}^k- H\beta_h^{k,p}(s,a)  \right\}\label{eq:p_t3}.
\end{align}
We are going to show that both~\eqref{eq:c_t3} and~\eqref{eq:p_t3} are less than zero for all $(k,h,s,a)$ with high probability by Lemma~\ref{lem:conc_po}. First, conditioned on the first result in Lemma~\ref{lem:conc_po}, we have~\eqref{eq:c_t3} is less than zero. Further, we have conditioned on the second result in Lemma~\ref{lem:conc_po}
\begin{align}
    &\widetilde{P}_h^{k}(\cdot|s,a)\widetilde{V}_h^k - P_h(\cdot|s,a)\widetilde{V}_{h+1}^k- H\beta_h^{k,p}(s,a) \nonumber \\
    \lep{a} &\norm{\widetilde{P}_h^{k}(\cdot|s,a) - P_h(\cdot|s,a)}_1 \norm{\widetilde{V}_{h+1}^k}_{\infty}- H\beta_h^{k,p}(s,a)\nonumber\\
    \lep{b} & H \norm{\widetilde{P}_h^{k}(\cdot|s,a) - P_h(\cdot|s,a)}_1- H\beta_h^{k,p}(s,a)\nonumber\\
    \lep{c} & 0\label{eq:conc_pv}
\end{align}
where (a) holds by Holder's inequality; (b) holds since $0\le \widetilde{V}_{h+1}^k \le H$ based on our update rule; (c) holds by Lemma~\ref{lem:conc_po}. Thus, we have shown that 
\begin{align}
    \mathcal{T}_3 \le 0.
\end{align}

\textbf{Analysis of $\mathcal{T}_1$}. Assume the good event $G$ and the event in Assumption~\ref{ass:rand} hold (which implies the concentration results in Lemma~\ref{lem:conc_po}). We have 
\begin{align}
    \mathcal{T}_1 = &\sum_{k=1}^K V_1^{\pi_k}(s_1) - \widetilde{V}_1^k(s_1)\nonumber\\
   \ep{a}& \sum_{k=1}^K\sum_{h=1}^H\ex{ c_h(s_h,a_h) + P_h(\cdot |s_h,a_h) \widetilde{V}_{h+1}^k - \widetilde{Q}_{h+1}^k(s_h,a_h)|s_1^k, \pi_k}\nonumber\\
   \ep{b}&\sum_{k=1}^K\sum_{h=1}^H\ex{ c_h(s_h,a_h) + P_h(\cdot |s_h,a_h) \widetilde{V}_{h+1}^k|s_1^k, \pi_k} \nonumber\\
  -&\sum_{k=1}^K\sum_{h=1}^H\ex{ \max\left \{ \widetilde{c}_h^k(s_h,a_h) - \beta_h^{k,c}(s_h,a_h)+ \widetilde{P}_h^k(\cdot|s_h,a_h)\widetilde{V}_{h+1}^k - H\beta_h^{k,p}(s_h,a_h) ,0\right\} |s_1^k, \pi_k}\label{eq:t1_po}
\end{align}
where (a) holds by the extended value difference lemma~\cite[Lemma 1]{efroni2020optimistic}; (b) holds by the update rue of Q-value in Algorithm~\ref{alg:PO}. Note that here we can directly remove the the truncation at $H$ since by Lemma~\ref{lem:conc_po}, $\widetilde{c}_h^k(s_h,a_h) - \beta_h^{k,c}(s_h,a_h)+ \widetilde{p}(\cdot|s_h,a_h)\widetilde{V}_{h+1}^k - H\beta_h^{k,p}(s_h,a_h) \le c_h(s,a) + P_h(\cdot|s,a) \widetilde{V}_{h+1}^k \le 1+H-h \le H$. 

Now, observe that for any $(k,h,s,a)$, we have 
\begin{align}
     &c_h(s,a) + P_h(\cdot |s,a) \widetilde{V}_{h+1}^k - \max\left \{ \widetilde{c}_h^k(s_h,a_h) - \beta_h^{k,c}(s,a)+ \widetilde{P}_h^k(\cdot|s,a)\widetilde{V}_{h+1}^k - H\beta_h^{k,p}(s,a) ,0\right\}\nonumber\\
     &\le c_h(s,a)-{\widetilde{c}_h^k(s,a)} + \beta_h^{k,c}(s,a) +  P_h(\cdot |s,a) \widetilde{V}_{h+1}^k- \widetilde{P}_h^k(\cdot |s,a) \widetilde{V}_{h+1}^k + H\beta_h^{k,p}(s,a)\nonumber\\
     &\lep{a} 2 \beta_h^{k,c}(s,a) + 2H \beta_h^{k,p}(s,a)\label{eq:t1_po_all},
\end{align}
where (a) holds by Lemma~\ref{lem:conc_po} and a similar analysis as in~\eqref{eq:conc_pv}. 
Plugging~\eqref{eq:t1_po_all} into~\eqref{eq:t1_po}, yields 
\begin{align}
    \mathcal{T}_1  \le \underbrace{\sum_{k=1}^K\sum_{h=1}^H\ex{ 2\beta_h^{k,c}(s_h,a_h)|s_1^k, \pi^k}}_{\text{Term(i)}} + \underbrace{H\sum_{k=1}^K\sum_{h=1}^H\ex{ 2\beta_h^{k,p}(s_h,a_h)|s_1^k, \pi^k}}_{\text{Term(ii)}}\label{eq:t1_final}
\end{align}
By the definition of $\beta_h^{k,c}$ and $\beta_h^{k,p}$ in Lemma~\ref{lem:conc_po} and Assumption~\ref{ass:rand}, we have with probability $1-2\delta$,
\begin{align*}
    \text{Term(i)} &\le  \sum_{k=1}^K\sum_{h=1}^H\ex{ \frac{L_c(\delta)}{\sqrt{\max\lbrace{N}_h^k(s_h,a_h),1\rbrace} }+\frac{3E_{\epsilon,\delta,1} }{\max\lbrace N_h^k(s_h,a_h),1\rbrace  }|s_1^k,\pi^k}\\
    \text{Term(ii)} &\le H\sum_{k=1}^K\sum_{h=1}^H \ex{\frac{L_p(\delta)}{\sqrt{\max\lbrace N_h^k(s_h,a_h),1\rbrace}} + \frac{SE_{\epsilon,\delta,2}+2 E_{\epsilon,\delta,1}}{\max\lbrace{N}_h^k(s_h,a_h),1\rbrace} | s_1^k, \pi^k}.
\end{align*}

We bound the two terms above by using the following lemma. These are generalization of results proved under stationary transition model  \cite{efroni2019tight,zanette2019tighter} to our non-stationary setting. The proof is given at the end of this section. 
\begin{lemma}
\label{lem:nonst}
With probability $1-2\delta$, we have 
\begin{align*}
    \sum_{k=1}^K\sum_{h=1}^H \ex{ \frac{1}{\max\{1,N_h^k(s_h,a_h)\}}|\mathcal{F}_{k-1}} = O\left(SAH\ln KH + H\ln(H/\delta)\right),
\end{align*}
and 
\begin{align*}
    \sum_{k=1}^K\sum_{h=1}^H \ex{ \frac{1}{\sqrt{\max\{1,N_h^k(s_h,a_h)\}} }|\mathcal{F}_{k-1}} = O\left(\sqrt{SAH^2K} + SAH\ln KH + H\ln(H/\delta)\right),
\end{align*}
where the filtration $\mathcal{F}_k$ includes all the events until the end of episode $k$.
\end{lemma}

Therefore, by Lemma~\ref{lem:nonst} (since $\pi^k$ is determined by $\mathcal{F}_{k-1}$), we have with probability at least $1-4\delta$,
\begin{align}
    \text{Term(i)} &\le  \sum_{k=1}^K\sum_{h=1}^H\ex{ \frac{L_c(\delta)}{\sqrt{\max\lbrace{N}_h^k(s_h,a_h),1\rbrace} }+\frac{3E_{\epsilon,\delta,1} }{\max\lbrace N_h^k(s_h,a_h),1\rbrace  }|s_1^k,\pi^k}\nonumber\\
    &\le \tilde{O}\left(\sqrt{SAH^2K} + SAH + E_{\epsilon,\delta,1}SAH  \right)\label{eq:term1}.
\end{align}
and 
\begin{align}
   \text{Term(ii)} &\le H\sum_{k=1}^K\sum_{h=1}^H \ex{\frac{L_p(\delta)}{\sqrt{\max\lbrace N_h^k(s_h,a_h),1\rbrace}} + \frac{SE_{\epsilon,\delta,2}+2 E_{\epsilon,\delta,1}}{\max\lbrace{N}_h^k(s_h,a_h),1\rbrace} | s_1^k, \pi^k}\nonumber\\
    &\le \tilde{O}\left(\sqrt{S^2AH^4K} + \sqrt{S^3A^2H^4} + E_{\epsilon,2}S^2AH^2 + E_{\epsilon,1}SAH^2\right)\label{eq:term2}.
\end{align}
Plugging~\eqref{eq:term1} and~\eqref{eq:term2} into~\eqref{eq:t1_final}, yields ($T = KH$)
\begin{align*}
    \mathcal{T}_1 = O\left(\left(\sqrt{S^2AH^3T} + \sqrt{S^3A^2H^4} + E_{\epsilon,\delta,2}S^2AH^2 +  E_{\epsilon,\delta,1}SAH^2\right) \log(S,A,T,1/\delta)\right)
\end{align*}
Finally, putting the bounds on $\mathcal{T}_1$, $\mathcal{T}_2$ and $\mathcal{T}_3$ together, completes the proof. 
\end{proof}

We are left to present the proof for Lemma~\ref{lem:nonst}. In the case of a stationary transition, \cite{efroni2019tight,zanette2019tighter} resort to the method of properly defining a `good' set of episodes (cf.~\cite[Definition 6]{zanette2019tighter}).  We prove our results in the non-stationary case via a different approach. In particular, inspired by~\cite{jin2020learning}, we will use the following Bernstein-type concentration inequality for martingale as our main tool, which is adapted from Lemma 9 in~\cite{jin2020learning}.
\begin{lemma}
\label{lem:MDS}
Let $Y_1,\ldots,Y_K$ be a martingale difference sequence with respect to a filtration $\mathcal{F}_0, \mathcal{F}_1, \ldots, \mathcal{F}_K$. Assume $Y_k \le R$ a.s. for all $i$. Then, for any $\delta \in (0,1)$ and $\lambda \in [0,1/R]$, with probability $1-\delta$, we have 
\begin{align*}
    \sum_{k=1}^K Y_k \le \lambda \sum_{k=1}^K\mathbb{E}\left[ Y_k^2|\mathcal{F}_{k-1}\right] + \frac{\ln(1/\delta)}{\lambda}.
\end{align*}
\end{lemma}

Now, we are well-prepared to present the proof of Lemma~\ref{lem:nonst}.
\begin{proof}[Proof of Lemma~\ref{lem:nonst}]
Let $\mathcal{I}_h^k(s,a)$ be the indicator so that $\mathbb{E}\left[\mathcal{I}_h^k(s,a) |\mathcal{F}_{k-1}\right] = w_h^k(s,a)$, which is the probability of visiting state-action pair $(s,a)$ at step $h$ and episode $k$. First note that 
    \begin{align*}
         &\sum_{k=1}^K\sum_{h=1}^H \ex{ \frac{1}{\max\{1,N_h^k(s_h,a_h)\}}|\mathcal{F}_{k-1}} \\
         = & \sum_{k=1}^K \sum_{h,s,a} w_h^k(s,a)  \frac{1}{\max\{1,N_h^k(s,a)\}}\\
         = & \sum_{k=1}^K\sum_{h,s,a}  \frac{\mathcal{I}_h^k(s,a)}{\max\{1,N_h^k(s,a)\}} + \sum_{k=1}^K\sum_{h,s,a}  \frac{w_h^k(s,a) - \mathcal{I}_h^k(s,a)}{\max\{1,N_h^k(s,a)\}}.
    \end{align*}
The first term can be bounded as follows.
\begin{align*}
    \sum_{k=1}^K\sum_{h,s,a}  \frac{\mathcal{I}_h^k(s,a)}{\max\{1,N_h^k(s,a)\}} &\le \sum_{h,s,a}  \sum_{k=1}^K \frac{1}{\max\{1,N_h^k(s,a)\}}\\
    & =\sum_{h,s,a} \sum_{i=1}^{N_h^K(s,a)} \frac{1}{i}\\
    &\le c'\sum_{h,s,a} \ln(N_h^K(s,a))\\
    &\le c' SAH\ln\left(\sum_{s,a,h} N_h^K(s,a) \right)\\
    &=O\left(SAH\ln(KH)\right).
\end{align*}
To bound the second term, we will use Lemma~\ref{lem:MDS}. In particular, consider $Y_{k,h} :=\sum_{s,a}  \frac{w_h^k(s,a) - \mathcal{I}_h^k(s,a)}{\max\{1,N_h^k(s,a)\}} \le 1$, $\lambda = 1$, and the fact that for any fixed $h$,
\begin{align*}
    \mathbb{E}\left[ Y_{k,h}^2|\mathcal{F}_{k-1}\right] &\le \mathbb{E}\left[ \left(\sum_{s,a} \frac{\mathcal{I}_h^k(s,a)}{\max\{1,N_h^k(s,a)\}}\right)^2\mid\mathcal{F}_{k-1}\right]\\
    &= \mathbb{E}\left[ \sum_{s,a} \frac{\mathcal{I}_h^k(s,a)}{\max\{1,(N_h^k(s,a))^2\}}\mid \mathcal{F}_{k-1}\right]\\
    &\le  \sum_{s,a} \frac{w_h^k(s,a)}{\max\{1,N_h^k(s,a)\}}.
\end{align*}
Then, via Lemma~\ref{lem:MDS}, we have with probability at least $1-\delta$, 
\begin{align*}
    \sum_{k=1}^K\sum_{h,s,a}  \frac{w_h^k(s,a) - \mathcal{I}_h^k(s,a)}{\max\{1,N_h^k(s,a)\}} = \sum_{h=1}^H\sum_{k=1}^K Y_{k,h} &\le \sum_{h=1}^H \sum_{k=1}^K\sum_{s,a} \frac{w_h^k(s,a)}{\max\{1,N_h^k(s,a)\}} + H\ln(H/\delta)\\
    & = O\left(SAH\ln(KH) + H\ln(H/\delta)\right),
\end{align*}
which completes the proof of the first result in Lemma~\ref{lem:nonst}. To show the second result, similarly, we decompose it as 
   \begin{align*}
         &\sum_{k=1}^K\sum_{h=1}^H \ex{ \frac{1}{\sqrt{\max\{1,N_h^k(s_h,a_h)\}} }|\mathcal{F}_{k-1}} \\
         = & \sum_{k=1}^K \sum_{h,s,a} w_h^k(s,a)  \frac{1}{\sqrt{\max\{1,N_h^k(s,a)\}} }\\
         = & \sum_{k=1}^K\sum_{h,s,a}  \frac{\mathcal{I}_h^k(s,a)}{\sqrt{\max\{1,N_h^k(s,a)\}} } + \sum_{k=1}^K\sum_{h,s,a}  \frac{w_h^k(s,a) - \mathcal{I}_h^k(s,a)}{\sqrt{\max\{1,N_h^k(s,a)\}} }
    \end{align*}
The first term can be bounded as follows.
\begin{align*}
    \sum_{k=1}^K\sum_{h,s,a}  \frac{\mathcal{I}_h^k(s,a)}{\max\{1,N_h^k(s,a)\}} &\le \sum_{h,s,a}  \sum_{k=1}^K \frac{1}{\sqrt{\max\{1,N_h^k(s,a)\}} }=\sum_{h,s,a} \sum_{i=1}^{N_h^K(s,a)} \frac{1}{\sqrt{i}}\le c'\sum_{h,s,a} \sqrt{N_h^K(s,a)}\\
    &\le c' \sqrt{ \left(\sum_{h,s,a} 1\right) \left( \sum_{h,s,a} N_h^K(s,a)\right) }\\
    &=O\left(\sqrt{SAH^2K }\right)
\end{align*}

To bound the second term, we apply Lemma~\ref{lem:MDS} again. Consider $Y_{k,h} :=\sum_{s,a}  \frac{w_h^k(s,a) - \mathcal{I}_h^k(s,a)}{\sqrt{\max\{1,N_h^k(s,a)\}}} \le 1$, $\lambda = 1$ and the fact that for any fixed $h$,
\begin{align*}
    \mathbb{E}\left[ Y_{k,h}^2|\mathcal{F}_{k-1}\right] &\le \mathbb{E}\left[ \left(\sum_{s,a} \frac{\mathcal{I}_h^k(s,a)}{\sqrt{\max\{1,N_h^k(s,a)\} }}\right)^2\mid\mathcal{F}_{k-1}\right]\\
    &= \mathbb{E}\left[ \sum_{s,a} \frac{\mathcal{I}_h^k(s,a)}{\max\{1,N_h^k(s,a)\}}\mid \mathcal{F}_{k-1}\right]\\
    &= \sum_{s,a} \frac{w_h^k(s,a)}{\max\{1,N_h^k(s,a)\}}.
\end{align*}

Then, via Lemma~\ref{lem:MDS}, we have with probability at least $1-\delta$, 
\begin{align*}
    \sum_{k=1}^K\sum_{h,s,a}  \frac{w_h^k(s,a) - \mathcal{I}_h^k(s,a)}{\sqrt{\max\{1,N_h^k(s,a)\}} } = \sum_{h=1}^H\sum_{k=1}^K Y_{k,h} &\le \sum_{h=1}^H \sum_{k=1}^K\sum_{s,a} \frac{w_h^k(s,a)}{\max\{1,N_h^k(s,a)\}} + H\ln(H/\delta)\\
    & = O\left(SAH\ln(KH) + H\ln(H/\delta)\right)
\end{align*}

Putting the two bounds together, yields the second result and completes the proof.
\end{proof}

\section{Proofs for Section~\ref{sec:vi}}
In this section, we presents proof for Lemma~\ref{lem:conc_vi} and Theorem~\ref{thm:VI}. We also discuss the gaps in the regret analysis of~\cite{vietri2020private} at the end of this section.  
To start with, we present the proof of Lemma~\ref{lem:conc_vi} as follows.
\begin{proof}[Proof of Lemma~\ref{lem:conc_vi}]
Assume the event in Assumption~\ref{ass:rand} hold. The first result in Lemma~\ref{lem:conc_vi} follows the same analysis as in the proof of Lemma~\ref{lem:conc_po}. In particular, we have with probability at least $1-\delta/2$,
\begin{align*}
    |c_h(s,a) - \widetilde{c}_h^k(s,a)| \le \beta_h^{k,c}(s,a).
\end{align*}

To show the second result, we first note that $V^*$ is fixed and $V^*_h(s) \le H$ for all $h$ and $s$. This enables us to use standard Hoeffding's inequality. Specifically, we have 
\begin{align*}
    &\left|(\widetilde{P}_h^k - P_h) {V}_{h+1}^*(s,a)\right|\\
    =&\left| \sum_{s'}\widetilde{P}_h^k(s'|s,a)V^*_{h+1}(s') - P_h(s'|s,a)V_{h+1}^*(s') \right|\\
    \le & \left|\sum_{s'} \left(\frac{{N}_h^k(s,a,s')}{(\widetilde{N}_h^k(s,a) +  E_{\epsilon,\delta,1})\vee 1} - P_h(s'|s,a) \right)V_{h+1}^*(s')  \right| + \left| \sum_{s'}  \frac{\widetilde{N}_h^k(s,a,s') - {N}_h^k(s,a,s')}{(\widetilde{N}_h^k(s,a) +  E_{\epsilon,\delta,1})\vee 1} V_{h+1}^*(s') \right|.
\end{align*}
We are going to bound the two terms respectively, For the first term, we have with probability at least $1-\delta/2$
\begin{align*}
    &\left|\sum_{s'} \left(\frac{{N}_h^k(s,a,s')}{(\widetilde{N}_h^k(s,a) +  E_{\epsilon,\delta,1})\vee 1} - P_h(s'|s,a) \right)V_{h+1}^*(s')  \right| \\
    \le & \left|  \sum_{s'} V_{h+1}^*(s')\left( \frac{N_h^k(s,a,s')}{N_h^k(s,a)\vee 1} - P_h(s'|s,a) \right) \frac{N_h^k(s,a)\vee 1}{(\widetilde{N}_h^k(s,a) +  E_{\epsilon,\delta,1})\vee 1}\right| \\
    &+\left|\sum_{s'} V_{h+1}^*(s') P_h(s'|s,a)\left(\frac{N_h^k(s,a)\vee 1}{(\widetilde{N}_h^k(s,a) +  E_{\epsilon,\delta,1})\vee 1}-1 \right) \right|\\
    \lep{a} &  \frac{N_h^k(s,a)\vee 1}{(\widetilde{N}_h^k(s,a) +  E_{\epsilon,\delta,1})\vee 1} H \sqrt{\frac{L(\delta)}{N_h^k(s,a)\vee 1}} \\
    &+ \left|\sum_{s'} V_{h+1}^*(s') P_h(s'|s,a)\left(\frac{N_h^k(s,a)\vee 1}{(\widetilde{N}_h^k(s,a) +  E_{\epsilon,\delta,1})\vee 1}-1 \right) \right|\\
    \le &  \frac{N_h^k(s,a)\vee 1}{(\widetilde{N}_h^k(s,a) +  E_{\epsilon,\delta,1})\vee 1} H \sqrt{\frac{L(\delta)}{N_h^k(s,a)\vee 1}}+ H\frac{2 E_{\epsilon,\delta,1}}{(\widetilde{N}_h^k(s,a) + E_{\epsilon,1})\vee 1}\\
    \le & H\sqrt{\frac{L(\delta)}{{(\widetilde{N}_h^k(s,a) +  E_{\epsilon,\delta,1} )\vee 1}}} + \frac{2 H E_{\epsilon,\delta,1}}{(\widetilde{N}_h^k(s,a) +  E_{\epsilon,\delta,1})\vee 1}
\end{align*}
where in (a) we use standard Hoeffding inequality with $L(\delta') := 2\ln\frac{4SAT}{\delta}$ and $V_{h+1}^{*}(s') \le H$. 

For the second term, we have 
\begin{align*}
    \left| \sum_{s'}  \frac{\widetilde{N}_h^k(s,a,s') - {N}_h^k(s,a,s')}{(\widetilde{N}_h^k(s,a) +  E_{\epsilon,\delta,1})\vee 1} V_{h+1}^*(s') \right| \le \frac{HS E_{\epsilon,\delta,2}}{(\widetilde{N}_h^k(s,a) +  E_{\epsilon,\delta,1})\vee 1}
\end{align*}

Putting the two bounds together, we have 
\begin{align*}
    &\left|(\widetilde{P}_h^k - P_h) {V}_{h+1}^*(s,a)\right|\le  H\sqrt{\frac{L(\delta')}{{(\widetilde{N}_h^k(s,a) +  E_{\epsilon,\delta,1} )\vee 1}}} +  \frac{H (S E_{\epsilon,\delta,2} + 2 E_{\epsilon,\delta,1})}{(\widetilde{N}_h^k(s,a) +  E_{\epsilon,\delta,1})\vee 1}
\end{align*}
Noting that $L_c(\delta) = \sqrt{L(\delta)}$, we have obtained the second result in Lemma~\ref{lem:conc_vi}.

Now, we turn to focus on the third result in Lemma~\ref{lem:conc_vi}. Note that 
    \begin{align*}
    &|P_h(s'|,s,a) - \widetilde{P}_h^k(s'|s,a)|\\
    =& \left| \frac{\widetilde{N}_h^k(s,a,s')}{(\widetilde{N}_h^k(s,a) +  E_{\epsilon,\delta,1})\vee 1} - P_h(s'|s,a)\right|\\
    \le& \underbrace{\left|  \frac{{N}_h^k(s,a,s')}{(\widetilde{N}_h^k(s,a) +  E_{\epsilon,\delta,1})\vee 1} - P_h(s'|s,a) \right|}_{\mathcal{P}_1} +  \underbrace{\left|  \frac{\widetilde{N}_h^k(s,a,s') - {N}_h^k(s,a,s')}{(\widetilde{N}_h^k(s,a) +  E_{\epsilon,\delta,1})\vee 1} \right|}_{\mathcal{P}_2}.
\end{align*}
For $\mathcal{P}_1$, we have with probability at least $1-\delta$,
    \begin{align*}
        &\mathcal{P}_1=\left|\frac{N_h^k(s,a,s')}{N_h^k(s,a)\vee 1}\frac{N_h^k(s,a)\vee 1}{(\widetilde{N}_h^k(s,a) + \alpha E_{\epsilon,\delta,1})\vee 1} - P_h(s'|s,a)\right|\\
        =&\left|\left( \frac{N_h^k(s,a,s')}{N_h^k(s,a)\vee 1} - P_h(s'|s,a) \right) \frac{N_h^k(s,a)\vee 1}{(\widetilde{N}_h^k(s,a) +  E_{\epsilon,\delta,1})\vee 1} + P_h(s'|s,a)\left(\frac{N_h^k(s,a)\vee 1}{(\widetilde{N}_h^k(s,a) +  E_{\epsilon,\delta,1})\vee 1}-1 \right) \right|\\
        \lep{a}& \frac{N_h^k(s,a)\vee 1}{(\widetilde{N}_h^k(s,a) +  E_{\epsilon,\delta,1})\vee 1} {|\bar{P}_h^k(s'|s,a) - P_h(s'|s,a) |} + \left( P_h(s'|s,a) \frac{2 E_{\epsilon,\delta,1}}{(\widetilde{N}_h^k(s,a) +  E_{\epsilon,\delta,1})\vee 1}\right)\\
        \lep{b}&C \frac{N_h^k(s,a)\vee 1}{(\widetilde{N}_h^k(s,a) +  E_{\epsilon,\delta,1})\vee 1} \left(\sqrt{\frac{L^{\prime}(\delta) P_h(s'|s,a)}{{N_h^k(s,a)\vee 1}}} + \frac{L^{\prime}(\delta)}{N_h^k(s,a)\vee 1}\right)+ \frac{2 E_{\epsilon,\delta,1}}{(\widetilde{N}_h^k(s,a) + E_{\epsilon,\delta,1})\vee 1}\\
        \le & C\sqrt{\frac{L^{\prime}(\delta)P_h(s'|s,a)}{{(\widetilde{N}_h^k(s,a) +  E_{\epsilon,\delta,1} )\vee 1}}} +C\frac{L^{\prime}(\delta)}{(\widetilde{N}_h^k(s,a) +  E_{\epsilon,\delta,1})\vee 1}+ \frac{2 E_{\epsilon,\delta,1}}{(\widetilde{N}_h^k(s,a) +  E_{\epsilon,\delta,1})\vee 1},
    \end{align*}
    where in (a) holds by $\bar{P}_h^k(\cdot|s,a)=\frac{N_h^k(s,a,s')}{N_h^k(s,a)\vee 1}$ and Assumption~\ref{ass:rand} for the second term; (b) holds by Lemma 8 in~\cite{efroni2020exploration} for some constant $C$ and $L^{\prime}(\delta) := \ln\left(\frac{6SAHK}{\delta}\right)$, which is an application of empirical Bernstein inequality (cf.~\cite[Theorem 4]{maurer2009empirical}).
    
  For $\mathcal{P}_2$, we have 
    \begin{align*}
        \mathcal{P}_2 \le \frac{E_{\epsilon,\delta,2}}{(\widetilde{N}_h^k(s,a) +  E_{\epsilon,\delta,1})\vee 1}
    \end{align*}
    
   Putting everything together, we obtain 
    \begin{align*}
        \sum_{s'} |P_h(s'|,s,a) - \widetilde{P}_h^k(s'|s,a)| \le   C\sqrt{\frac{L^{\prime}(\delta)P_h(s'|s,a)}{{(\widetilde{N}_h^k(s,a) +  E_{\epsilon,\delta,1} )\vee 1}}} +\frac{CL^{\prime}(\delta) + 2 E_{\epsilon,\delta,1}
        + E_{\epsilon,\delta,2} }{(\widetilde{N}_h^k(s,a) +  E_{\epsilon,\delta,1})\vee 1}.
    \end{align*}
Finally, applying union bound over all the events, yields the required results in Lemma~\ref{lem:conc_vi}.
\end{proof}

Now, we turn to establish Theorem~\ref{thm:VI}. First, the next lemma establishes that the value function maintained in our algorithm is optimistic.

\begin{lemma}
\label{lem:opt}
Fix $\delta \in (0,1]$, with probability at least $1-3\delta$, $\widetilde{V}_h^k(s) \le  V_h^*(s)$ for all $(k,h,s)$.
\end{lemma}
\begin{proof}
   For a fixed $k$, consider $h=H+1,H,\ldots, 1$. In the base case $h=H+1$, it trivially holds since $\widetilde{V}_{H+1}^k(s) = 0 = V_{H+1}^*(s)$. Assume that  $\widetilde{V}_h^k(s) \le V_h^*(s)$ for all $s$. Then, by the update rule, we have 
   \begin{align*}
       \widetilde{{Q}}_h^k(s,a) = \min\{H, \max\{0,\widetilde{c}_h^k(s,a) + (\widetilde{P}_h^k\widetilde{V}_{h+1}^k)(s,a) - \beta_h^k(s,a)\} \}
   \end{align*}
   
   First, we would like to show that the truncation at $H$ does not affect the analysis. To see this, first observe that under Lemma~\ref{lem:conc_vi}
   \begin{align}
       \widetilde{c}_h^k(s,a) + (\widetilde{P}_h^k\widetilde{V}_{h+1}^k)(s,a)- \beta_h^k(s,a) &\lep{a} c_h(s,a) + (\widetilde{P}_h^k\widetilde{V}_{h+1}^k)(s,a) - \beta_{h}^{k,pv}(s,a)\nonumber\\
       &\lep{b} c_h(s,a) + (\widetilde{P}_h^k{V}_{h+1}^*)(s,a) - \beta_{h}^{k,pv}(s,a)\nonumber\\
       &\lep{c} c_h(s,a) + (P_hV_{h+1}^*)(s,a) = Q_h^*(s,a) \le H \label{eq:relation}
   \end{align}
   where (a) holds by the first result in Lemma~\ref{lem:conc_vi}; (b) holds by induction; (c) holds by the second result in Lemma~\ref{lem:conc_vi}. This directly implies that 
   \begin{align*}
       \widetilde{{Q}}_h^k(s,a) = \max\{0,\widetilde{c}_h^k(s,a) + (\widetilde{P}_h^k\widetilde{V}_{h+1}^k)(s,a) - \beta_h^k(s,a)\}
   \end{align*}
    Hence, if the maximum is attained at zero, then $\widetilde{{Q}}_h^k(s,a) \le Q_h^*(s,a)$ trivially holds since $Q_h^*(s,a) \in [0,H]$. Otherwise, by Eq.~\eqref{eq:relation}, we also have $\widetilde{{Q}}_h^k(s,a) \le Q_h^*(s,a)$. Therefore, we have $\widetilde{Q}_h^k(s,a) \le Q_h^*(s,a)$, and hence $\widetilde{V}_h^k(s) \le V_h^*(s)$.
\end{proof}

Based on the result above, we are now able to present the proof of Theorem~\ref{thm:VI}.
\begin{proof}[Proof of Theorem~\ref{thm:VI}]
By the optimistic result in Lemma~\ref{lem:opt}, we have 
\begin{align}
\label{eq:vi_regret_start}
    \mathcal{R}(K) = \sum_{k=1}^K (V_1^{\pi_k}(s_1) - V_1^{*}(s_1)) \le \sum_{k=1}^K({V}_1^{\pi_k}(s_1) - \widetilde{V}_1^{k}(s_1))
\end{align}
Now, we turn to upper bound $V_h^{\pi_k}(s_h^k)-\widetilde{V}_h^k(s_h^k)$ by a recursive form. First, observe that 
\begin{align*}
    (V_h^{\pi_k}-\widetilde{V}_h^k )(s_h^k)  = (Q_h^{\pi_k}-\widetilde{Q}_h^k )(s_h^k,a_h^k),
\end{align*}
which holds since the action executed by $\pi_k$ at step $h$, and the action used to update $\widetilde{V}_h^k$ is the same. Now, to bound the $Q$-value difference, we have 
\begin{align}
\label{eq:q_decom}
    &(Q_h^{\pi_k}-\widetilde{Q}_h^k )(s_h^k,a_h^k)\nonumber \\
    \lep{a}& 2\beta_h^{k,c}(s_h^k,a_h^k) + (P_h V_{h+1}^{\pi_k}-\widetilde{P}_h^k \widetilde{V}_{h+1}^k )(s_h^k,a_h^k) + \beta_h^{k,pv}(s,a)\nonumber\\
    = &\left[(P_h-\widetilde{P}_h^k ) \widetilde{V}_{h+1}^k\right](s_h^k,a_h^k) + \left[ P_h(V_{h+1}^{\pi_k}-\widetilde{V}_{h+1}^k )\right](s_h^k,a_h^k) + 2\beta_h^{k,c}(s_h^k,a_h^k) + \beta_h^{k,pv}(s,a)\nonumber\\
    = & \left[( P_h-\widetilde{P}_h^k ) {V}_{h+1}^*\right](s_h^k,a_h^k) + \left[(\widetilde{P}_h^k-P_h ) ( {V}_{h+1}^*-\widetilde{V}_{h+1}^k)\right](s_h^k,a_h^k)\nonumber \\
    &+ \left[ P_h(V_{h+1}^{\pi_k}-\widetilde{V}_{h+1}^k )\right](s_h^k,a_h^k) + 2\beta_h^{k,c}(s_h^k,a_h^k) + \beta_h^{k,pv}(s,a)\nonumber\\
    \lep{b}&\left[(P_h-\widetilde{P}_h^k ) ({V}_{h+1}^*- \widetilde{V}_{h+1}^k)\right](s_h^k,a_h^k) +\left[ P_h( V_{h+1}^{\pi_k}-\widetilde{V}_{h+1}^k )\right](s_h^k,a_h^k)\nonumber \\
    &+ 2\beta_h^{k,c}(s_h^k,a_h^k) + 2\beta_h^{k,pv}(s_h^k,a_h^k),
\end{align}
where (a) we have used the cost concentration result in Lemma~\ref{lem:conc_vi}; (b) holds by the transition concentration result in  Lemma~\ref{lem:conc_vi}.
Thus, so far we have arrived at 
\begin{align}
     (V_h^{\pi_k}-\widetilde{V}_h^k )(s_h^k) & \le \left[(P_h-\widetilde{P}_h^k ) ({V}_{h+1}^*- \widetilde{V}_{h+1}^k)\right](s_h^k,a_h^k) +\left[ P_h( V_{h+1}^{\pi_k}-\widetilde{V}_{h+1}^k )\right](s_h^k,a_h^k)\nonumber \\
    &+ 2\beta_h^{k,c}(s_h^k,a_h^k) + 2\beta_h^{k,pv}(s_h^k,a_h^k)\label{eq:vi_recur_start}.
\end{align}
We will first carefully analyze the first term. In particular, let $G:= ({V}_{h+1}^*-\widetilde{V}_{h+1}^k )$, we have 
\begin{align}
&\left[(P_h-\widetilde{P}_h^k ) (\widetilde{V}_{h+1}^k- {V}_{h+1}^*)\right](s_h^k,a_h^k)\nonumber\\
    =&\sum_{s'}\left( P_h(s'|s_h^k,a_h^k)-\widetilde{P}_h^k(s'|s_h^k,a_h^k) \right)G(s')\nonumber\\
    \lep{a} & c \sum_{s'} \left( \sqrt{\frac{L^{\prime}(\delta)P_h(s'|s_h^k,a_h^k)}{{(\widetilde{N}_h^k(s,a) +  E_{\epsilon,\delta,1} )\vee 1}}} +\frac{L^{\prime}(\delta)}{(\widetilde{N}_h^k(s,a) +  E_{\epsilon,\delta,1})\vee 1}+ \frac{2 E_{\epsilon,\delta,1} + E_{\epsilon,\delta,2}}{(\widetilde{N}_h^k(s,a) +  E_{\epsilon,\delta,1})\vee 1}\right) G(s')\nonumber\\
    \lep{b} & \sum_{s'} \left(\frac{P_h(s'|s_h^k,a_h^k)}{H} G(s')  \right)+ c \sum_{s'} \left(\frac{cH L^{\prime}(\delta)} {(\widetilde{N}_h^k(s,a) +  E_{\epsilon,\delta,1})\vee 1}\right)G(s')\nonumber \\
    &+ c\sum_{s'} \left(\frac{L^{\prime}(\delta)}{(\widetilde{N}_h^k(s,a) +  E_{\epsilon,\delta,1})\vee 1}\right)G(s') + c\sum_{s'}\left(\frac{2 E_{\epsilon,\delta,1} + E_{\epsilon,\delta,2}}{(\widetilde{N}_h^k(s,a) +  E_{\epsilon,\delta,1})\vee 1}\right)G(s')\nonumber\\
    \lep{c} & \sum_{s'} \left(\frac{P_h(s'|s_h^k,a_h^k)}{H}  G(s')\right) +  \sum_{s'} \left(\frac{c'H^2 L^{\prime}(\delta)} {(\widetilde{N}_h^k(s,a) +  E_{\epsilon,\delta,1})\vee 1}\right)\nonumber\\
    &+c\sum_{s'}\left(\frac{2H E_{\epsilon,\delta,1} + HE_{\epsilon,\delta,2}}{(\widetilde{N}_h^k(s,a) +  E_{\epsilon,\delta,1})\vee 1}\right)\label{eq:vi_recur_temp},
\end{align}
where (a) holds by the third result in Lemma~\ref{lem:conc_vi} and $c$ is some absolute constant; (b) holds by $\sqrt{xy} \le x + y$ for positive numbers $x,y$; (c) holds since $G(s') \le 2H$ by the boundedness of $V$-value. Now, plugging the definition for $G(s')$ into~\eqref{eq:vi_recur_temp}, yields
\begin{align}
&\left[(P_h-\widetilde{P}_h^k ) ({V}_{h+1}^*-\widetilde{V}_{h+1}^k )\right](s_h^k,a_h^k)\nonumber\\
    =&\frac{1}{H} \left[P_h(V_{h+1}^*-\widetilde{V}_{h+1}^k )\right](s_h^k,a_h^k) + \frac{c'SH^2 L^{\prime}(\delta)} {(\widetilde{N}_h^k(s,a) +  E_{\epsilon,\delta,1})\vee 1} + \frac{2cSH E_{\epsilon,\delta,1} + cSHE_{\epsilon,\delta,2}}{(\widetilde{N}_h^k(s,a) +  E_{\epsilon,\delta,1})\vee 1}\nonumber\\
    \ep{a}&\frac{1}{H} \left[P_h(V_{h+1}^*-\widetilde{V}_{h+1}^k )\right](s_h^k,a_h^k) + \xi_h^k + \zeta_h^k\nonumber\\
    \lep{b}& \frac{1}{H}\left[P_h(V_{h+1}^{\pi_k}-\widetilde{V}_{h+1}^k )\right](s_h^k,a_h^k)+\xi_h^k + \zeta_h^k\label{eq:vi_recur_end},
\end{align}
where (a) holds by definitions $\xi_h^k=:\frac{c'SH^2 L^{\prime}(\delta)} {(\widetilde{N}_h^k(s,a) +  E_{\epsilon,\delta,1})\vee 1}$ and $\zeta_h^k:=\frac{2cSH E_{\epsilon,\delta,1} + cSHE_{\epsilon,\delta,2}}{(\widetilde{N}_h^k(s,a) +  E_{\epsilon,\delta,1})\vee 1}$; (b) holds since $V_{h+1}^{\pi_k} \ge V_{h+1}^*$. Plugging~\eqref{eq:vi_recur_end} into~\eqref{eq:vi_recur_start}, yields the following recursive formula. 
\begin{align*}
      (V_h^{\pi_k}-\widetilde{V}_h^k )(s_h^k) &\lep{a} \left(1+\frac{1}{H}\right)\left[P_h( V_{h+1}^{\pi_k}-\widetilde{V}_{h+1}^k)\right](s_h^k,a_h^k)+\xi_h^k + \zeta_h^k + 2\beta_h^k\\
     &\ep{b} \left(1+\frac{1}{H}\right)\left[( V_{h+1}^{\pi_k}-\widetilde{V}_{h+1}^k)(s_{h+1}^k) + \chi_{h}^k\right] + \xi_h^k + \zeta_h^k + 2\beta_h^k
\end{align*}
where in (a), we let $\beta_h^k:= \beta_h^{k,r}(s_h^k,a_h^k) + \beta_h^{k,pv}(s_h^k,a_h^k)$ for notation simplicity; (b) holds by definition $\chi_h^k:=\left[P_h( V_{h+1}^{\pi_k}-\widetilde{V}_{h+1}^k)\right](s_h^k,a_h^k) - ( V_{h+1}^{\pi_k}-\widetilde{V}_{h+1}^k)(s_{h+1}^k)$. Based on this, we have the following bound on $(\widetilde{V}_1^k - V_1^{\pi_k})(s_1^k)$,
\begin{align}
     &(\widetilde{V}_1^k - V_1^{\pi_k})(s_1^k)\nonumber\\
     \le & \left(1+\frac{1}{H}\right)(\chi_{1}^k + \xi_1^k + \zeta_1^k + 2\beta_1^k) + \left(1+\frac{1}{H}\right)^2(\chi_{2}^k + \xi_2^k + \zeta_2^k + 2\beta_2^k) + \ldots\nonumber \\
     &+ \left(1+\frac{1}{H}\right)^H(\chi_{H}^k + \xi_H^k + \zeta_H^k + 2\beta_H^k)\nonumber\\
     \le & 3\sum_{h=1}^H (\chi_{h}^k + \xi_h^k + \zeta_h^k + 2\beta_h^k)\label{eq:vi_regret_end}.
\end{align}
Therefore, plugging~\eqref{eq:vi_regret_end} into~\eqref{eq:vi_regret_start}, we have the regret decomposition as follows.
\begin{align*}
    \mathcal{R}(K) \le 3\sum_{k=1}^K\sum_{h=1}^H (\chi_{h}^k + \xi_h^k + \zeta_h^k + 2\beta_h^k)
\end{align*}
We are only left to bound each of them. To start with, we focus on the bonus term. We focus on $\beta_{h}^{k,pv}(s,a)$ in particular since it upper bounds the term  $\beta_{h}^{k,c}(s,a)$ as shown in Lemma~\ref{lem:conc_vi}. By definition, we have 
\begin{align*}
    &\sum_{k=1}^K\sum_{h=1}^H \beta_h^{k,pv}(s_h^k,a_h^k) \\
    =&\underbrace{H\sum_{k=1}^K\sum_{h=1}^H  \frac{L_c(\delta)} {\sqrt{{(\widetilde{N}_h^k(s_h^k,a_h^k) +  E_{\epsilon,\delta,1} )\vee 1}}}   }_{\mathcal{T}_1}+  \underbrace{H\sum_{k=1}^K\sum_{h=1}^H\frac{ (S E_{\epsilon,\delta,2} + 2 E_{\epsilon,\delta,1})}{(\widetilde{N}_h^k(s_h^k,a_h^k) +  E_{\epsilon,\delta,1})\vee 1}}_{\mathcal{T}_2}.
\end{align*}
The first term can be upper bounded as follows ($T:= KH$) under Assumption~\ref{ass:rand}.
\begin{align*}
    \mathcal{T}_1 &\le H{L_c(\delta)}\sum_{k=1}^K\sum_{h=1}^H \sqrt{\frac{1}{{{N}_h^k(s,a) \vee 1}}}\\
    &=H{L_c(\delta)} \sum_{h,s,a}\sum_{i=1}^{N_h^K(s,a)} \frac{1}{\sqrt{i}}\\
    &\le c' H{L_c(\delta)}\sum_{h,s,a} \sqrt{N_h^K(s,a)}\\
    &\le c' H{L_c(\delta)}\sqrt{ \left(\sum_{h,s,a} 1\right) \left( \sum_{h,s,a} N_h^K(s,a)\right) }\\
    &=\widetilde{O}\left(\sqrt{H^3SAT}\right).
\end{align*}
The second term can be upper bounded as follows under Assumption~\ref{ass:rand}. 
\begin{align*}
    \mathcal{T}_2 &\le c H(SE_{\epsilon,\delta,2} + E_{\epsilon,\delta,1})\sum_{k=1}^K\sum_{h=1}^H\frac{1}{N_h^k(s_h^k,a_h^k)
    \vee 1}\\
    &= c H(SE_{\epsilon,\delta,2} + E_{\epsilon,\delta,1}) \sum_{h,s,a}\sum_{i=1}^{N_h^K(s,a)} \frac{1}{i}\\
    &\le c'H(SE_{\epsilon,\delta,2} + E_{\epsilon,\delta,1}) \sum_{h,s,a}\ln(N_h^K(s,a))\\
    &= \widetilde{O}\left(H^2S^2AE_{\epsilon,\delta,2} + H^2SAE_{\epsilon,\delta,1}\right).
\end{align*}
Putting them together, we have the following bound on the summation over $\beta_h^k$.
\begin{align*}
    \sum_{k=1}^K\sum_{h=1}^H \beta_h^k = \widetilde{O}\left(\sqrt{H^3SAT} + H^2S^2AE_{\epsilon,\delta,2} +  H^2SAE_{\epsilon,\delta,1}\right).
\end{align*}
By following the same analysis as in $\mathcal{T}_2$, we can bound the summation over $\xi_h^k=:\frac{c'SH^2 L^{\prime}(\delta)} {(\widetilde{N}_h^k(s,a) +  E_{\epsilon,\delta,1})\vee 1}$ and $\zeta_h^k:=\frac{2cSH E_{\epsilon,\delta,1} + cSHE_{\epsilon,\delta,2}}{(\widetilde{N}_h^k(s,a) +  E_{\epsilon,\delta,1})\vee 1}$ as follows.
\begin{align*}
    &\sum_{k=1}^K\sum_{h=1}^H\xi_h^k = \widetilde{O}\left( H^3S^2A\right)\\
    &\sum_{k=1}^K\sum_{h=1}^H\zeta_h^k = \widetilde{O}\left(H^2S^2A(E_{\epsilon,2} + E_{\epsilon,1})\right).
\end{align*}

Finally, we are going to bound the summation over $\chi_h^k:=\left[P_h( V_{h+1}^{\pi_k}-\widetilde{V}_{h+1}^k)\right](s_h^k,a_h^k) - ( V_{h+1}^{\pi_k}-\widetilde{V}_{h+1}^k)(s_{h+1}^k)$, which turns out to be a martingale difference sequence. In particular, we define a filtration $\mathcal{F}_h^k$ that includes all the randomness up to the $k$-th episode and the $h$-th step. Then, we have $\mathcal{F}_1^1 \subset \mathcal{F}_2^1 \ldots \subset \mathcal{F}_H^1\subset \mathcal{F}_1^2\subset \mathcal{F}_2^2 \ldots $. Also, we have $(\widetilde{V}_{h+1}^k - V_{h+1}^{\pi_k}) \in \mathcal{F}_1^k \subset \mathcal{F}_h^k$ since they are decided by data collected up to episode $k-1$. A bit abuse of notation, we define $X_{h+1}^k :=\chi_h^k$. Then, we have 
\begin{align}
    \mathbb{E}\left[X_{h+1}^k |\mathcal{F}_h^k\right] = 0.
\end{align}
This holds since the expectation only captures randomness over $s_{h+1}^k$. Thus, $X_{h+1}^k$ is a martingale difference sequence. Moreover, we have $|X_{h+1}^k| \le 4H$ a.s. By Azuma-Hoeffding inequality, we have with probability at least $1-\delta$
\begin{align*}
    \sum_{k=1}^K\sum_{h=1}^H \chi_h^k = \sum_{k=1}^K\sum_{h=1}^H X_{h+1}^k = c'\sqrt{H^2T\ln(2/\delta)} = \widetilde{O}\left(\sqrt{H^2T} \right)
\end{align*}

Putting everything together, and applying union bound on all high-probability events,  we have shown that with probability at least $1-\delta$,
\begin{align}
\label{eq:rt_vi_proof}
    \mathcal{R}(T) = O\left( \left(\sqrt{SAH^3T} + S^2AH^3 + S^2AH^2E_{\epsilon,\delta,1} + S^2AH^2E_{\epsilon,\delta,2}\right){\log}(S,A,T,1/\delta)\right).
\end{align}
\end{proof}

\subsection{Discussions}
We end this section by comparing our results with existing works on private value-iteration RL, i.e.,~\cite{garcelon2020local} on LDP and~\cite{vietri2020private} on JDP.
\begin{itemize}
    \item In~\cite{garcelon2020local}, the privacy-independent leading term has a dependence on $S$ rather than $\sqrt{S}$ in our result (i.e., the first term in~\eqref{eq:rt_vi_proof}). This is because they directly bound the transition term $\norm{P_h(\cdot|s,a) - \widetilde{P}_h^k(\cdot|s,a)}_1$, which incurs the additional $\sqrt{S}$. In particular, after step (a) in~\eqref{eq:q_decom}, they directly bound $\widetilde{P}_h^k \widetilde{V}_{h+1}^k$ by $P_h \widetilde{V}_{h+1}^k + H\beta_h^{k,p}$ and then recursively expand the term. Note that $\beta_h^{k,p}$ has an additional factor $\sqrt{S}$, which directly leads to the dependence $S$ in the final result. In contrast, we handle (a) in~\eqref{eq:q_decom} by following the idea in~\cite{azar2017minimax}. That is, we first extract the term $(P_h - \widetilde{P}_h^k)V_{h+1}^*$, which can be bounded by standard Hoeffding's inequality since $V_{h+1}^*$ is fixed and hence no additional $\sqrt{S}$ is introduced. Due to this extraction, we have an additional `correction' term, i.e., $\left[(P_h-\widetilde{P}_h^k ) (\widetilde{V}_{h+1}^k- {V}_{h+1}^*)\right]$. To bound it, we use Bernstein’s-type inequality to bound $(P_h(s'|s,a) - \widetilde{P}_h^k(s'|s,a))$ in~\eqref{eq:vi_recur_temp}. This allows us to obtain the final recursive formula. 
    \item In~\cite{vietri2020private}, although the claimed result has the same regret bound as ours, its current analysis has gaps. First, to derive the regret decomposition in Lemma 18 therein, the private estimates were incorrectly used as the true cost and transition functions. This lead to a simpler but incorrect regret decomposition since it omits the ‘error’ term between the private estimates and true values. Second, even if we add the omitted `error' term (between private estimates and true values) into the regret decomposition, its current analysis cannot achieve the same result as ours. This is due to a similar argument in bullet one. That is, in order to use its current confidence bound $\widetilde{\text{conf}}_t$ to avoid the additional $\sqrt{S}$ factor, it needs to use Bernstein’s-type inequality to bound the `correction' term. They fail to consider this since the regret decomposition does not have the `error' term as it was incorrectly omitted in Lemma 18.
\end{itemize}

\section{Proofs for Section~\ref{sec:privacy}}
First, we present the proof of Lemma \ref{lem:central}. Corollary \ref{cor:JDP} is a direct consequence of Lemma \ref{lem:central} and Theorems \ref{thm:PO} and \ref{thm:VI}.

\begin{proof}[Proof of Lemma \ref{lem:central}]
We start with the privacy guarantee of the {\central}. First, consider the counters for number of visited states $N_h^k(s,a)$. Note that there are $SAH$ many counters, and each counter is a $K$-bounded binary mechanism of  \citet{chan2010private}.
Now, consider the counter corresponding to a fixed tuple $(s,a,h) \in \cS \times \cA \times [H]$. Note that, at every episode $k \in [K]$, the private count $\widetilde N_h^k(s,a)$ is the sum of at most $\log K$ noisy \psums, where each \psum\ is corrupted by an independent Laplace noise $\text{Lap}\left(\frac{3H\log K}{\epsilon}\right)$. Therefore, by \cite[Theorem 3.5]{chan2010private}, the private counts $\lbrace\widetilde N_h^k(s,a)\rbrace_{k \in [K]}$ satisfy $\frac{\epsilon}{3H}$-DP. 

Now, observe that each counter takes as input the data stream $\sigma_h(s,a)\in \lbrace 0,1\rbrace^K$, where the $j$-th bit $\sigma_h^j(s,a):=\indic{s_h^j=s,a_h^j=a}$ denotes whether the pair $(s,a)$ is encountered or not at step $h$ of episode $j$. Consider, some other data stream $\sigma'_h(s,a)\in \lbrace 0,1\rbrace^K$ which differs from $\sigma_h(s,a)$ only in one entry. Then, we have $\norm{\sigma_h(s,a)-\sigma'_h(s,a)}_1 = 1$. Furthermore, since at every episode at most $H$ state-action pairs are encountered, we obtain
\begin{equation*}
    \sum_{(s,a,h) \in \cS \times \cA \times [H]}\norm{\sigma_h(s,a)-\sigma'_h(s,a)}_1 \le H,
\end{equation*}
Therefore, by \cite[Lemma 34]{hsu2016private}, the composition of all these $SAH$ different counters, each of which is $\frac{\epsilon}{3H}$-DP, satisfies $\frac{\epsilon}{3}$-DP.

Using similar arguments, one can show that composition of the counters for empirical rewards $C_h^k(s,a)$ and state transitions $N_h^k(s,a,s')$ satisfy $\frac{\epsilon}{3}$-DP individually. Finally, employing the composition property of DP \cite{dwork2014algorithmic}, we obtain that the {\central} is $\epsilon$-DP.

Let us now focus on the utility of the {\central}. First, fix a tuple $(s,a,h) \in \cS \times \cA \times [H]$, and consider the private counts $\widetilde N_h^k(s,a)$ corresponding to number of visited states $N_h^k(s,a)$.
Note that, at each episode $k \in [K]$, the cost of privacy $\abs{\widetilde{N}_{h}^k(s,a) - {N}_h^k(s,a)}$ is the sum of at most $\log K$ i.i.d. random variables $\text{Lap}\left(\frac{3H\log K}{\epsilon}\right)$. Therefore, by \cite[Theorem 3.6]{chan2010private}, we have
\begin{equation*}
    \prob{\abs{\widetilde{N}_{h}^k(s,a) - {N}_h^k(s,a)} \le  \frac{3H}{\epsilon}\sqrt{8\log^3 K\log(6/\delta)}} \ge 1-\delta/3.
\end{equation*}
Now, by a union bound argument, we obtain
\begin{equation*}
    \prob{\forall (s,a,k,h),\;\;\abs{\widetilde{N}_{h}^k(s,a) - {N}_h^k(s,a)} \le  \frac{3H}{\epsilon}\sqrt{8\log^3 K\log(6SAT/\delta)}} \ge 1-\delta/3.
\end{equation*}
Using similar arguments, one can show that the private counts $\widetilde C_h^k(s,a)$ and $\widetilde N_h^k(s,a,s')$ corresponding rewards $C_h^k(s,a)$ and state transitions $N_h^k(s,a,s')$, respectively, satisfy
\begin{align*}
    &\prob{\forall (s,a,k,h),\;\;\abs{\widetilde{C}_{h}^k(s,a) - {C}_h^k(s,a)} \le  \frac{3H}{\epsilon}\sqrt{8\log^3 K\log(6SAT/\delta)}} \ge 1-\delta/3,\\
    &\prob{\forall (s,a,s',k,h),\;\;\abs{\widetilde{N}_{h}^k(s,a,s') - {N}_h^k(s,a,s')} \le  \frac{3H}{\epsilon}\sqrt{8\log^3 K\log(6S^2AT/\delta)}} \ge 1-\delta/3.
\end{align*}
Combining all the three guarantees together using a union bound, we obtain that {\central} satisfies Assumption~\ref{ass:rand}.
\end{proof}

Next, we present the proof of Lemma \ref{lem:local}. Corollary \ref{cor:LDP} is a direct consequence of Lemma \ref{lem:local} and Theorems \ref{thm:PO} and \ref{thm:VI}.

\begin{proof}[Proof of Lemma \ref{lem:local}]
We start with the utility guarantee of the {\local}. First, fix a tuple $(s,a,h) \in \cS \times \cA \times [H]$, and consider the private counts
$\widetilde N_h^k(s,a)$ for the number of visited states $N_h^k(s,a)$. Note that, at each episode $k \in [K]$, the cost of privacy $\abs{\widetilde{N}_{h}^k(s,a) - {N}_h^k(s,a)}$ is the sum of at most $K$ i.i.d. random variables $\text{Lap}\left(\frac{3H}{\epsilon}\right)$.
Therefore, by \cite[Corollary 12.4]{dwork2014algorithmic}, we have
\begin{equation*}
    \prob{\abs{\widetilde{N}_{h}^k(s,a) - {N}_h^k(s,a)} \le  \frac{3H}{\epsilon} \sqrt{8K\log(6/\delta)}} \ge 1-\delta/3.
\end{equation*}
Now, by a union bound argument, we obtain
\begin{equation*}
    \prob{\forall (s,a,k,h),\;\;\abs{\widetilde{N}_{h}^k(s,a) - {N}_h^k(s,a)} \le  \frac{3H}{\epsilon}\sqrt{8K\log(6SAT/\delta)}} \ge 1-\delta/3,
\end{equation*}
Using similar arguments, one can show that the private counts $\widetilde C_h^k(s,a)$ and $\widetilde N_h^k(s,a,s')$ corresponding rewards $C_h^k(s,a)$ and state transitions $N_h^k(s,a,s')$, respectively, satisfy
\begin{align*}
    &\prob{\forall (s,a,k,h),\;\;\abs{\widetilde{C}_{h}^k(s,a) - {C}_h^k(s,a)} \le  \frac{3H}{\epsilon}\sqrt{8 K\log(6SAT/\delta)}} \ge 1-\delta/3,\\
    &\prob{\forall (s,a,s',k,h),\;\;\abs{\widetilde{N}_{h}^k(s,a,s') - {N}_h^k(s,a,s')} \le  \frac{3H}{\epsilon}\sqrt{8 K\log(6S^2AT/\delta)}} \ge 1-\delta/3.
\end{align*}
Combining all the three guarantees together using a union bound, we obtain that {\local} satisfies Assumption~\ref{ass:rand}.

Now, we turn towards the privacy guarantee of the {\local}. First, we fix an episode $k \in [K]$. Now, we fix a tuple $(s,a,h) \in \cS \times \cA \times [H]$, and consider the private version $\widetilde\sigma_h^k(s,a)$ of the bit $\sigma_h^k(s,a) \in \lbrace 0,1 \rbrace$, where $\sigma_h^k(s,a)$ denotes whether the pair $(s,a)$ is encountered or not at step $h$ of episode $k$. Note that $\widetilde\sigma_h^k(s,a)$ is obtained from $\sigma_h^k(s,a)$ via a Laplace mechanism with noise level $\frac{3H}{\epsilon}$. Since, the sensitivity of the input function is $1$, the Laplace mechanism is $\frac{\epsilon}{3H}$-DP, as well as it is $\frac{\epsilon}{3H}$-LDP \cite{dwork2014algorithmic}. Furthermore, since at every episode at most $H$ state-action pairs are encountered, by \cite[Lemma 34]{hsu2016private}, the composition of all these $SAH$ different Laplace mechanisms are $\epsilon/3$-LDP.

Using similar arguments, one can show that composition of the corresponding Laplace mechanisms for empirical rewards and state transitions satisfy $\frac{\epsilon}{3}$-LDP individually. Finally, employing the composition property of DP \cite{dwork2014algorithmic}, we obtain that the {\local} is $\epsilon$-LDP.
\end{proof}

\end{appendix}

\end{document}